\documentclass[10pt,letterpaper]{amsart}

\usepackage{amsfonts, amsmath,amssymb,array,latexsym, hyperref, amscd}

\usepackage{fancyhdr,a4wide}
\usepackage{amsthm}
\usepackage{mathtools} 
\usepackage{subcaption}
\usepackage{xcolor}
\usepackage{comment}
\usepackage{multirow}
\usepackage{booktabs}
\usepackage{varwidth}
\usepackage{pifont}
\usepackage{makecell}
\usepackage[flushleft]{threeparttable}
\usepackage{enumitem}
\usepackage{algorithm}
\usepackage{algorithmic}
\usepackage[vlined,linesnumbered,ruled,algo2e]{algorithm2e}
\usepackage[linesnumbered,ruled,vlined]{algorithm2e}
\usepackage{palatino}

\newtheorem{theorem}{Theorem}[section]

\newtheorem{lemma}[theorem]{Lemma}
\newtheorem{claim}[theorem]{Claim}
\newtheorem{corollary}[theorem]{Corollary}
\newtheorem{proposition}[theorem]{Proposition}

\numberwithin{equation}{section}

\newcommand{\uline}[1]{\underline{#1}}

\newcommand{\norm}[1]{\left|\left|#1\right|\right|}
\newcommand{\lr}[1]{\left(#1\right)}
\newcommand{\abs}[1]{\left|#1\right|}
\newcommand{\set}[1]{\left\{#1\right\}}

\newcommand{\R}{\mathbb R}

\newcommand{\gives}{\rightarrow}

\newcommand{\mD}{\mathcal D}

\newcommand{\mL}{\mathcal L}
\newcommand{\mI}{\mathcal I}

\newcommand{\rel}[1]{\left[#1\right]_+}
\newcommand{\RR}{\mathrm{RidgelessReLU}(\mD)}
\renewcommand{\sin}[1]{s_{\mathrm{in}}(x_{#1})}
\newcommand{\sout}[1]{s_{\mathrm{out}}(x_{#1})}
\newcommand{\PL}{\mathrm{PL}}
\newcommand{\PLD}{\mathrm{PL(\mD)}}
\newcommand{\TV}[1]{\norm{#1}_{TV}}
\newcommand{\I}[2]{(x_{#1},x_{#2})}
\newcommand{\sigin}{\sigma_{\mathrm{in}}}
\newcommand{\sigout}{\sigma_{\mathrm{out}}}
\newcommand{\sigstar}{\sigma_{*}}

\title{Ridgeless Interpolation with Shallow ReLU Networks in $1D$ is Nearest Neighbor Curvature Extrapolation and Provably Generalizes on Lipschitz Functions}

\begin{document}

\maketitle
\begin{center}
\author{Boris Hanin\footnote{BH gratefully acknowledges support from NSF grants DMS -- 1855684 and DMS -- 2133806 as well as from an ONR MURI on Foundations of Deep Learning}\\ Department of Operations Research and Financial Engineering\\ Princeton University}
\end{center}

\begin{abstract}
    We prove a precise geometric description of all one layer ReLU networks $z(x;\theta)$ with a single linear unit and input/output dimensions equal to one that interpolate a given dataset $\mD=\set{(x_i,f(x_i))}$ and, among all such interpolants, minimize the $\ell_2$-norm of the neuron weights. Such networks can intuitively be thought of as those that minimize the mean-squared error over $\mD$ plus an infinitesimal weight decay penalty. We therefore refer to them as ridgeless ReLU interpolants. Our description proves that, to extrapolate values $z(x;\theta)$ for inputs $x\in (x_i,x_{i+1})$ lying between two consecutive datapoints, a ridgeless ReLU interpolant simply compares the signs of the discrete estimates for the curvature of $f$ at $x_i$ and $x_{i+1}$ derived from the dataset $\mD$. If the curvature estimates at $x_i$ and $x_{i+1}$ have different signs, then $z(x;\theta)$ must be linear on $(x_i,x_{i+1})$. If in contrast the curvature estimates at $x_i$ and $x_{i+1}$ are both positive (resp. negative), then $z(x;\theta)$ is convex (resp. concave) on  $(x_i,x_{i+1})$. Our results show that ridgeless ReLU interpolants achieve the best possible generalization for learning $1d$ Lipschitz functions, up to universal constants. 
\end{abstract}
\section{Introduction}
The ability of overparameterized neural networks to simultaneously fit data (i.e. interpolate) and generalize to unseen data (i.e. extrapolate) is a robust empirical finding that spans the use of deep learning in tasks from computer vision \cite{krizhevsky2012imagenet,he2016deep}, natural language processing \cite{brown2020language}, and reinforcement learning \cite{silver2016mastering,vinyals2019alphastar,jumper2021highly}. This observation is surprising when viewed from the lens of traditional learning theory \cite{vapnik1971uniform,bartlett2002rademacher}, which advocates for capacity control of model classes and strong regularization to avoid overfitting. 

Part of the difficulty in explaining conceptually why neural networks are able to generalize is that it is unclear how to understand, concretely in terms of the network function, various forms of implicit and explicit regularization used in practice. For example, a well-chosen initialization for gradient-based optimizers is key to ensuring good generalization properties of the resulting learned network \cite{mishkin2015all,he2015delving, xiao2018dynamical}. However, the specific geometric or analytic properties of the learned network ensured by a successful initialization scheme are hard to pin down. 

In a similar vein, it is standard practice to experiment with explicit regularizers such as weight decay, obtained by adding an $\ell_2$ penalty on model parameters to the underlying empirical risk. While weight decay is easy to describe via its effect on parameters, it is typically challenging to translate this into properties of a learned non-linear model. In the simple setting of one layer ReLU networks there has been some relatively recent progress in this direction. Specifically, starting with an observation in \cite{neyshabur2014search} the articles \cite{savarese2019infinite, ongie2019function, parhi2020banach, parhi2020neural, parhi2021kinds} explore and develop the fact that $\ell_2$ regularization on parameters in this setting is provably equivalent to penalizing the total variation of the derivative of the network function (cf eg Theorem \ref{T:prior} from prior work below). While the results in these articles hold for any input dimension, in this article we consider the simplest case of input dimension $1$. In this setting, our main contributions are:\\
\begin{enumerate}
    \item Given a dataset $\mD = \set{(x_i,y_i)}$ with scalar inputs and outputs, we obtain a complete characterization of all one layer ReLU networks with a single linear unit which fit the data and, among all such interpolating networks, do so with the minimal $\ell_2$ norm of the neuron weights. There are infinitely many such networks and, unlike in prior work, our characterization is phrased directly in terms of the behavior of the network function on intervals $(x_i,x_{i+1})$ between consecutive datapoints. Our description is purely geometric and can be summarized informally as follows (see Theorem \ref{T:interpolants} for the precise statement):\\
    \begin{itemize}
        \item If we order $x_1<\cdots<x_m$, then the data itself gives a discrete curvature estimate 
        \[
        \epsilon_i := \mathrm{sgn}\lr{s_i-s_{i-1}},\qquad s_i:=\frac{y_{i+1}-y_i}{x_{i+1}-x_i}
        \]
        at $x_i$ of whatever function generated the data. See Figure \ref{fig:informal}. \\
        \item If the curvature estimates $\epsilon_i$ and $\epsilon_{i+1}$ at $x_i$ and $x_{i+1}$ disagree, then the network must be linear on $(x_i,x_{i+1})$. See Figures \ref{fig:alg1} and \ref{fig:alg2}.\\
        \item If the curvature estimates agree and are positive (resp. negative), then the network function is convex (resp. concave) on $(x_i,x_{i+1})$ and lies below (resp. above) the straight line interpolant of the data. See Figures \ref{fig:alg1} and \ref{fig:alg2}.\\
    \end{itemize}
    \item The geometric description of the space of interpolants of $\mD$ from (1)  immediately yields sharp generalization bounds for learning $1d$ Lipschitz functions. This is stated in Corollary \ref{C:lip-gen}. Specifically, if the dataset $\mD$ is generated by setting $y_i=f_*(x_i)$ for $f_*:\R\gives \R$ a Lipschitz function, then any one layer ReLU network with a single linear unit which interpolates $\mD$ but does so with minimal $\ell_2$-norm of the network parameters will generalize as well as possible to unseen data, up to a universal multiplicative constant. To the author's knowledge this is the first time such generalization guarantees have been obtained.\\
\end{enumerate}
\begin{figure}
    \centering
    \includegraphics[width=.9\textwidth]{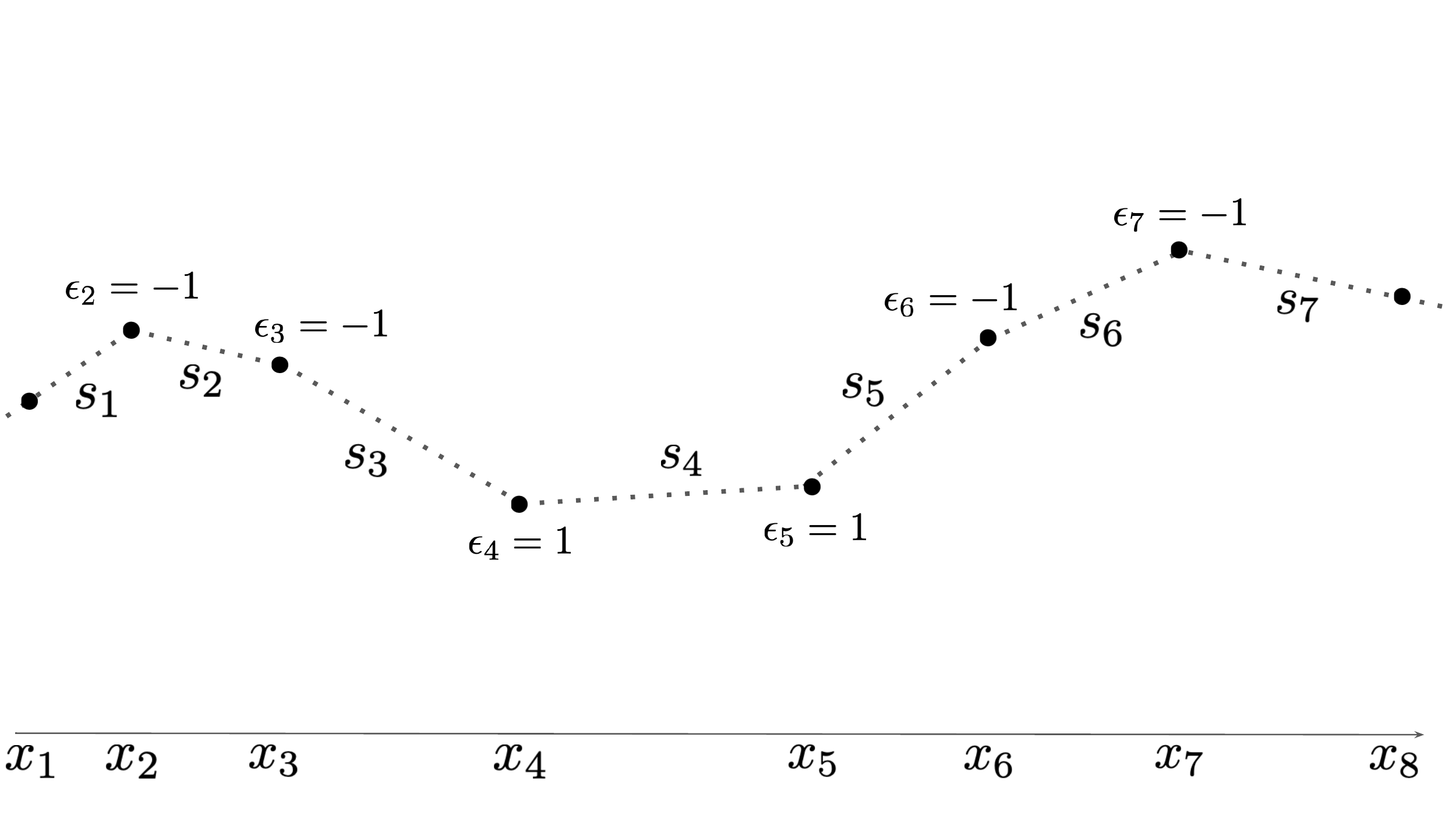}
    \caption{A dataset $\mD$ with $m=8$ points. Shown are the ``connect the dots'' interpolant $f_{\mD}$ (dashed line), its slopes $s_i$ and the ``discrete curvature'' $\epsilon_i$ at each $x_i$.}
    \label{fig:informal}
\end{figure}
\subsection{Setup and Informal Statement of Results}
\noindent Consider a one layer ReLU network
\begin{equation}\label{E:z-def}
z(x)=z(x;\theta):=ax + b+\sum_{j=1}^{n} W_j^{(2)}\rel{W_j^{(1)}x+b_i^{(1)}},\quad [t]_+:=\mathrm{ReLU}(t)=\max\set{0,t}    
\end{equation}
with a single linear unit\footnote{The presence of the linear term $ax+b$ is not really standard in practice but is adopted in keeping with prior work \cite{savarese2019infinite,ongie2019function,parhi2020banach} since it leads a cleaner mathematical formulation of results.} and input/output dimensions equal to one. For a given dataset
\[
\mD = \set{\lr{x_i,y_i},\, i=1,\ldots,m},\qquad -\infty<x_1<\cdots<x_m<\infty,\quad y_i\in \R,
\]
if the number of datapoints $m$ is smaller than the network width $n$, there are infinitely many choices of the parameter vector $\theta$ for which $z(x;\theta)$ interpolates (i.e. fits) the data:
\begin{equation}\label{E:interp}
z(x_i;\theta)=y_i,\qquad \forall\, i=1,\ldots,m.    
\end{equation}
Without further information about how $\theta$ was selected, little can be said about the  function $x\mapsto z(x;\theta)$ on intervals $(x_i,x_{i+1})$ between two consecutive datapoints when $n$ is much larger than $m$. This precludes useful generalization guarantees that hold uniformly over all $\theta$ subject only to the interpolation condition \eqref{E:interp}. 

In practice interpolants are not chosen arbitrary. Instead, they are typically learned by some variant of gradient descent starting from a random initialization. For a given network architecture, initialization scheme, optimizer, data augmentation scheme,  regularizer, and so on, understanding how the learned network uses the known labels $\set{y_i,\, i=1,\ldots,m}$ to extrapolate values of $z(x;\theta)$ for $x$ in intervals  $(x_i,x_{i+1})$ away from the datapoints in $\mD$ is an important open problem. To obtain non-trivial generalization estimates and make progress on this problem, a fruitful line of inquiry in prior work has been to search for additional complexity measures based on margins \cite{wei2018margin}, PAC-Bayes estimates \cite{dziugaite2017computing,dziugaite2018data,nagarajan2019deterministic}, weight matrix norms \cite{neyshabur2015norm, bartlett2017spectrally}, information theoretic compression estimates \cite{arora2018stronger}, Rachemacher complexity \cite{golowich18size}, etc that, while perhaps not explicitly regularized, are hopefully small in trained networks. The idea is then that these complexity measures being small gives additional constrains on the capacity of the space of learned networks. We refer the interested reader to \cite{jiang2019fantastic} for a review and empirically comparison of many such approaches.

In this article, we take a different approach to studying generalization. We do not seek general results that are valid for any network architecture. Instead, our goal is to describe completely, in concrete geometrical terms, the properties of one layer ReLU networks $z(x;\theta)$ that interpolate a dataset $\mD$ in the sense of \eqref{E:interp} with the minimal possible $\ell_2$ penalty
\[
C(\theta)=C(\theta,n)=\frac{1}{2}\sum_{j=1}^{n}\lr{\abs{W_j^{(1)}}^2 + \abs{W_j^{(2)}}^2}
\]
on the neuron weights. More precisely, we study the space of ridgeless ReLU interpolants 
\begin{equation}\label{E:RR-def}
\mathrm{RidgelessReLU}(\mD):=\set{z(x;\theta)\quad|\quad z(x_i;\theta)=y_i\,\,\,\, \forall (x_i,y_i)\in \mD,\quad C(\theta)=C_*},
\end{equation}
of a dataset $\mD$, where
\[
C_* := \inf_{\theta,n} \set{C(\theta,n)~|~z(x_i;n,\theta)=y_i\,\,\,\,\forall (x_i,y_i)\in \mD}.
\]
The elements of $\RR$ can intuitively be thought of as all ReLU networks that minimize a weakly penalized loss
\begin{equation}\label{E:loss-reg}
\mL(\theta;\mD)+\lambda C(\theta),\qquad\lambda \ll 1,    
\end{equation}
where $\mL$ is an empirical loss, such as the mean squared error over $\mD$, and the strength $\lambda$ of the weight decay penalty $C(\theta)$ is infinitesimal. It it plausible but by no means obvious that, with high probability, gradient descent from a random initialization and a weight decay penalty whose strength decreases to zero over training converges to an element in $\RR$. This article does not study optimization, and we therefore leave this as an interesting open problem. Our main result is simple description of $\RR$ and can informally be stated as follows:
\begin{theorem}[Informal Statement of Theorem \ref{T:interpolants}]\label{T:interpolants-informal}
Fix a dataset $\mD = \set{(x_i,y_i),\, i=1,\ldots,m}$. Each datapoint $(x_i,y_i)$ gives an estimate
\[
\epsilon_i := \mathrm{sgn}\lr{s_i-s_{i-1}},\qquad s_i:=\frac{y_{i+1}-y_{i}}{x_{i+1}-x_i}
\]
for the local curvature of the data (Figure \ref{fig:informal}). Among all continuous and piecewise linear functions $f$ that fit $\mD$ exactly, the ones in $\RR$ are precisely those that:\\
\begin{itemize}
    \item Are convex (resp. concave) on intervals $(x_i,x_{i+1})$ at which neighboring datapoints agree on the local curvature in the sense that $\epsilon_i=\epsilon_{i+1}=1$ (resp. $\epsilon_{i}=\epsilon_{i+1}=-1)$. On such intervals $f$ lies below (resp. above) the straight line interpolant of the data. See Figures \ref{fig:alg1} and \ref{fig:alg2}.\\
    \item Are linear (or more precisely affine) on intervals $(x_{i},x_{i+1})$ when neighboring datapoints disagree on the local curvature in the sense that $\epsilon_i\cdot \epsilon_{i+1}\neq 1$. \\
\end{itemize}
\end{theorem}
Before giving a precise statement our results, we mention that, as described in detail below, the space $\RR$ has been considered in a number of prior articles \cite{savarese2019infinite,ongie2019function,parhi2020banach}. Our starting point will be the useful but abstract characterization of $\RR$ they obtained in terms of the total variation of the derivative of $z(x;\theta)$ (see \eqref{E:rr-old}). 

Let us also note that the conclusions of Theorem \ref{T:interpolants-informal} (and Theorem \ref{T:interpolants}) also hold under seemingly very different hypotheses from ours. Namely, instead of $\ell_2$-regularization on the parameters, \cite{blanc2020implicit} considers SGD training for mean squared error with iid noise added to labels. Their Theorem 2 shows (modulo some assumptions about interpreting the derivative of the ReLU) that, among all ReLU networks a linear unit that interpolate a dataset $\mD$, the only ones that minimize the implicit regularization induced by adding iid noise to SGD are precisely those that satisfy the conclusions of Theorem \ref{T:interpolants-informal} and hence are exactly the networks in $\RR$. This suggests that our results hold under much more general conditions. It would be interesting to characterize them.

Further, our characterization of $\RR$ in Theorem \ref{T:interpolants} immediately implies strong generalization guarantees uniformly over $\RR$. We give a representative example in Corollary \ref{C:lip-gen}, which shows that such ReLU networks achieve the best possible generalization error of Lipschitz functions, up to constants. %As we point out just below Corollary \ref{C:lip-gen}, such generalization results can alternatively be obtained, albeit with worse constants, by what appears to be a novel localization of the representation \eqref{E:rr-old} of $\RR$ derived in \cite{savarese2019infinite,ongie2019function}. 

Finally, note that we allow networks $z(x;\theta)$ of any width but that if the width $n$ is too small relative to the dataset size $m$, then the interpolation condition \eqref{E:interp} cannot be satisfied. Also, we point out that in our formulation of the cost $C(\theta)$ we have left both the linear term $ax+b$ and the neuron biases  unregularized. This is not standard practice but seems to yield the cleanest results.

\subsection{Statement of Results and Relation to Prior Work}
Every ReLU network $z(x;\theta)$ is a continuous and piecewise linear function from $\R$ to $\R$ with a finite number of affine pieces. Let us denote by $\PL$ the space of all such functions and define
\[
\PLD:=\set{f\in \PL|~ f(x_i)=y_i\,\, \forall i=1,\ldots, m }
\]
to be the space of piecewise linear interpolants of $\mD$. Perhaps the most natural element in $\PLD$ is the ``connect-the-dots interpolant'' $f_{\mD}:\R\gives \R$ given by
\[
f_{\mD}(x) := \begin{cases}
\ell_1(x),&\quad x < x_2\\
\ell_i(x),&\quad x_i<x<x_{i+1},\quad i=2,\ldots,m-2\\
\ell_{m-1}(x),&\quad x > x_{m-1}
\end{cases},
\]
where for $i=1,\ldots, m-1$, we've set
\[
\ell_i(x):=(x-x_i)s_i + y_i,\qquad s_i:=\frac{y_{i+1}-y_i}{x_{i+1}-x_i}.
\]
See Figure \ref{fig:informal}. In addition to $f_{\mD}$, there are many other elements in $\RR$. Theorem \ref{T:interpolants} gives a complete description of all of them phrased in terms of how they may behave on intervals $(x_i,x_{i+1})$ between consecutive datapoints. Our description is based on the signs
\[
\epsilon_i = \mathrm{sgn}\lr{s_{i}-s_{i-1}},\qquad 2\leq i \leq m
\]
of the (discrete) second derivatives of $f_{\mD}$ at the inputs $x_i$ from our dataset.

\begin{theorem}\label{T:interpolants}
The space $\mathrm{RidgelessReLU}(\mD)$ consists of those $f\in \mathrm{PL(\mD)}$ satisfying:\\
\begin{enumerate}
    \item $f$ coincides with $f_{\mD}$ on the following intervals:\\
    \begin{itemize}
        \item[(1a)] Near infinity, i.e. on the intervals $(-\infty, x_2)$, $(x_{m-1},\infty)$\\
        \item[(1b)] Near datapoints that have zero discrete curvature, i.e. on intervals $(x_{i-1},x_{i+1})$ with $i=2,\ldots, m-1$ such that $\epsilon_i=0$.\\
        \item[(1c)] Between datapoints with opposite discrete curvature, i.e. on intervals $(x_i,x_{i+1})$ with $i=2,\ldots, m-1$ such that $\epsilon_i\cdot \epsilon_{i+1}=-1$.\\
        \end{itemize}
    \item $f$ is convex (resp. concave) and bounded above (resp. below) by $f_{\mD}$ between any consecutive datapoints at which the discrete curvature is positive (resp. negative). Specifically, suppose for some $3\leq i\leq i+q\leq m-2$ that $x_i$ and $x_{i+q}$ are consecutive discrete inflection points in the sense that 
    \[
    \epsilon_{i-1}\neq \epsilon_i,\qquad \epsilon_i=\cdots = \epsilon_{i+q},\qquad \epsilon_{i+q}\neq \epsilon_{i+q+1}.
    \]
    If $\epsilon_i=1$ (resp. $\epsilon_i=-1$), then restricted to the interval $(x_i,x_{i+q})$, $f$ is convex (resp. concave) and lies above (resp. below) the incoming and outgoing support lines and below (resp. above) $f_{\mD}$:
    \begin{align*}
    \epsilon_i&=1\qquad \Longrightarrow \qquad \max\set{\ell_{i-1}(x),\, \ell_{i+q}(x)}~\leq~ f(x) \leq f_{\mD}(x)\\
    \epsilon_i &= -1 \qquad \Longrightarrow \qquad     \min\set{\ell_{i-1}(x),\, \ell_{i+q}(x)}~\geq~ f(x)~\geq ~f_{\mD}(x)
    \end{align*}
    for all $x\in (x_i,x_{i+q})$.
\end{enumerate}
\end{theorem}

\begin{figure}
\centering
\vspace{-2cm}
\begin{subfigure}{\linewidth}
    \centering
    \includegraphics[width=.9\linewidth]{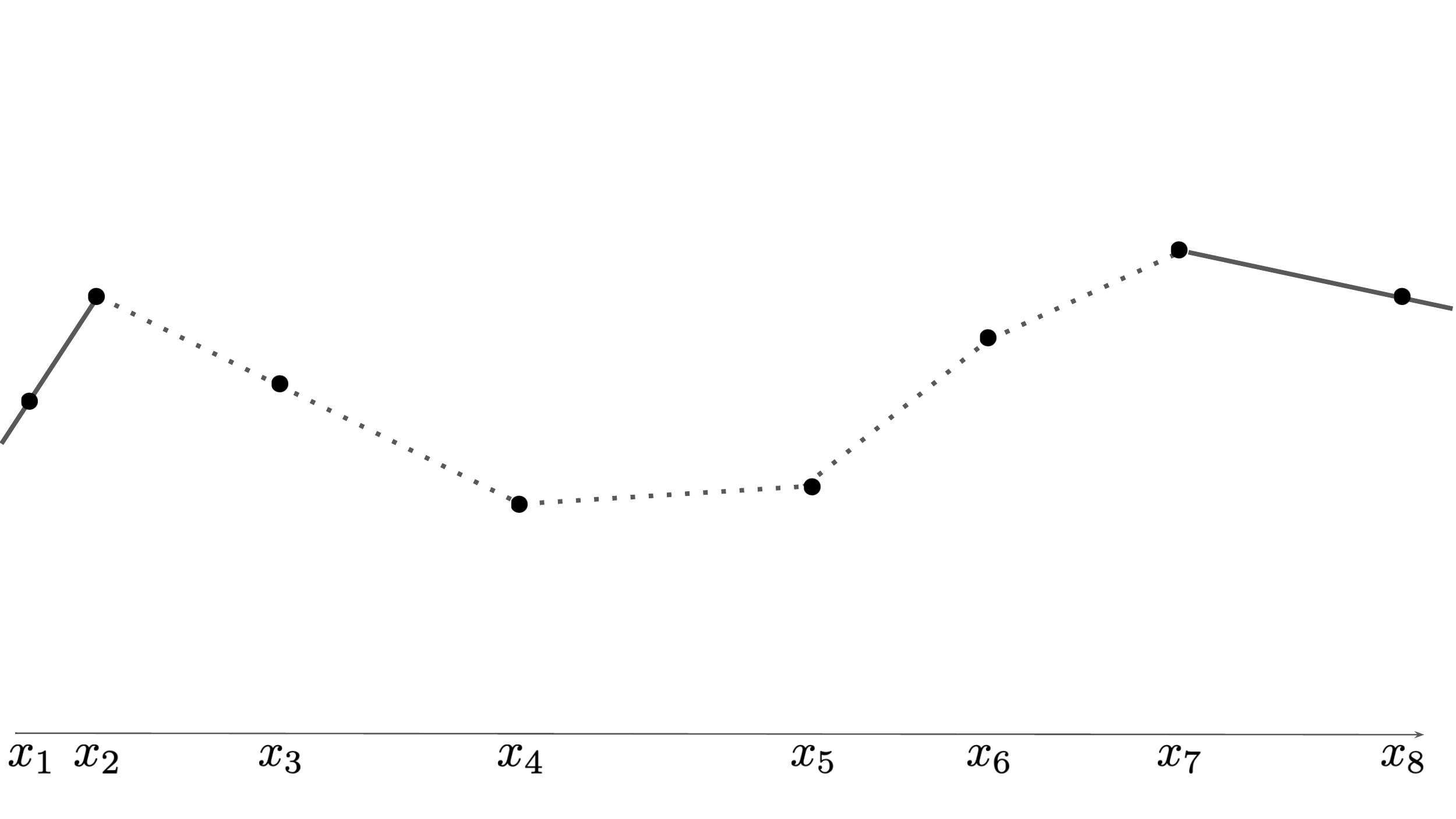}
    \caption{Step 1}\label{fig:image1}
\end{subfigure}

\begin{subfigure}{\linewidth}
    \centering
     \includegraphics[width=.9\linewidth]{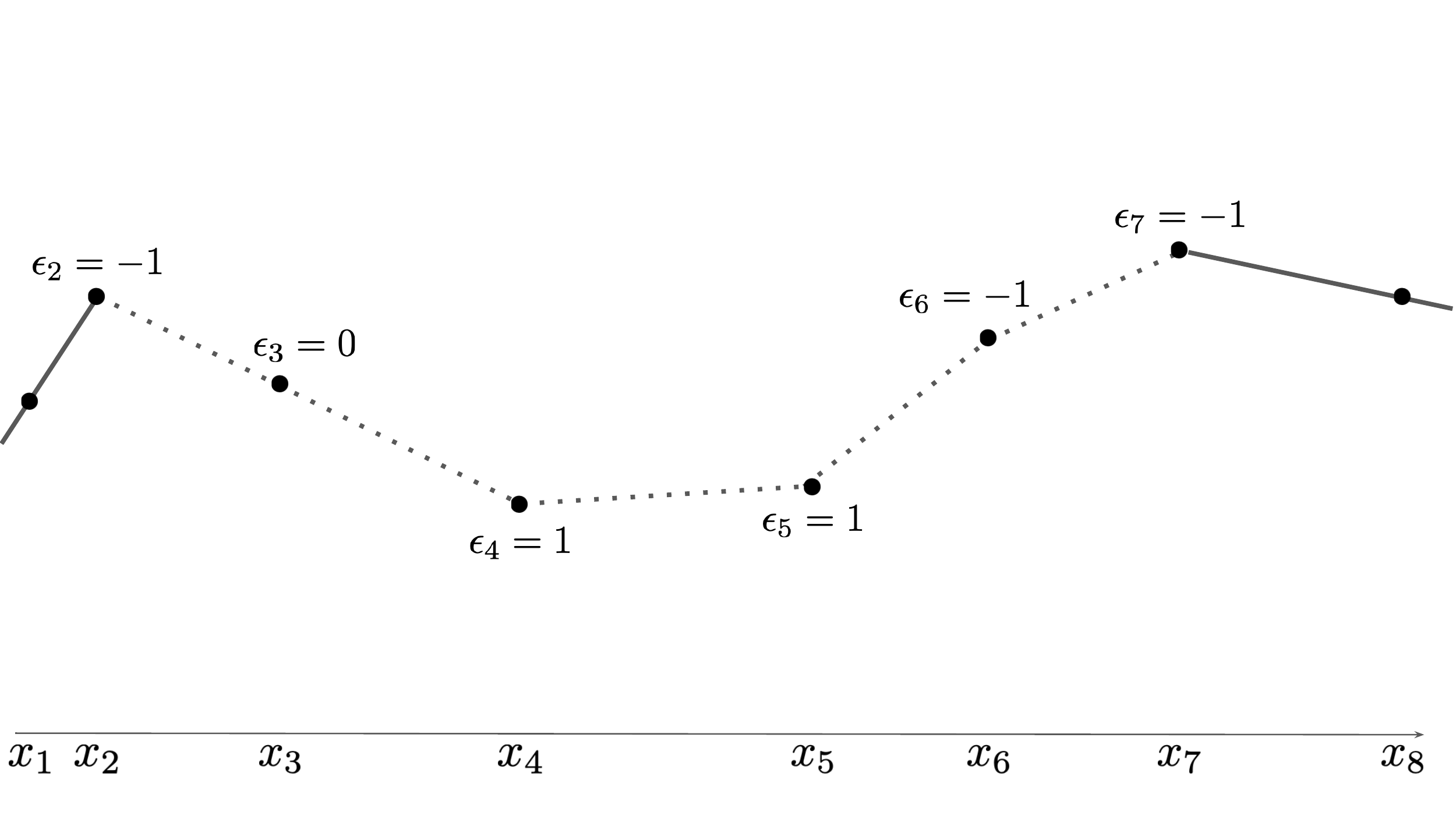}
    \caption{Step 2}\label{fig:image12}
\end{subfigure}

\begin{subfigure}{\linewidth}
    \centering
     \includegraphics[width=.9\linewidth]{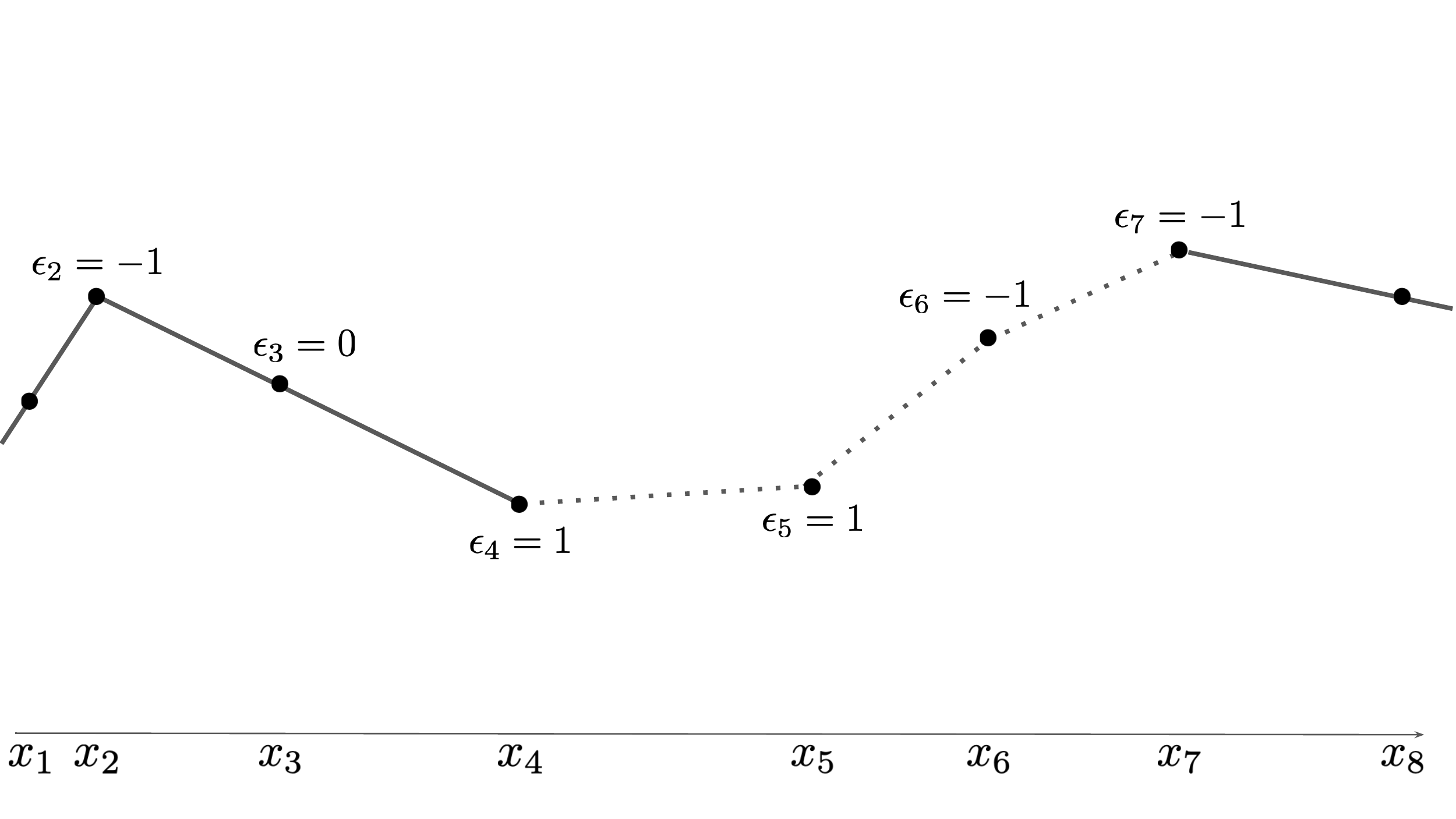}
    \caption{Step 3}\label{fig:image13}
\end{subfigure}
\caption{Steps 1 - 3 for generating $\RR$ from the dataset $\mD$.}
\label{fig:alg1}
\end{figure}

We refer the reader to \S \ref{S:pf} for a proof of Theorem \ref{T:interpolants}. Before doing so, let us illustrate Theorem \ref{T:interpolants} as an algorithm that, given the dataset $\mD$, describes all elements in $\RR$ (see Figures \ref{fig:alg1} and \ref{fig:alg2}): \\
\begin{enumerate}
    \item[\uline{Step 1}] \textbf{Linearly interpolate the endpoints}: by property (1), $f\in \RR$ must agree with $f_{\mD}$ on $(-\infty,x_2)$ and $(x_{m-1},\infty)$. \\
    \item[\uline{Step 2}] \textbf{Compute discrete curvature}: for $i=2,\ldots, m-1$ calculate the discrete curvature $\epsilon_i$ at the data point $x_i$.\\
    \item[\uline{Step 3}] \textbf{Linearly interpolate on intervals with zero curvature}: for all $i=2,\ldots,m-1$ at which $\epsilon_i=0$ property (1) guarantees that $f$ coincides with the $f_{\mD}$ on $(x_{i-1},x_{i+1})$. \\
    \item[\uline{Step 4}] \textbf{Linearly interpolate on intervals with ambiguous curvature}: for all $i=2,\ldots,m-1$ at which $\epsilon_i\cdot \epsilon_{i+1}=-1$ property (1) guarantees that $f$ coincides with $f_{\mD}$ on $(x_{i},x_{i+1})$.\\
    \item[\uline{Step 5}] \textbf{Determine convexity/concavity on remaining points}: all intervals $(x_i,x_{i+1})$ on which $f$ has not yet been determined occur in sequences $(x_i,x_{i+1}),\ldots, (x_{i+q-1},x_{i+q})$ on which $\epsilon_{i+j}=1$ or $\epsilon_{i+j}=1$ for all $j=0,\ldots,q$. If $\epsilon_i=1$ (resp. $\epsilon_i=-1$), then $f$ is any convex (resp. concave) function bounded below (resp. above) by $f_{\mD}$ and above (resp. below) the support lines $\ell_{i}(x),\, \ell_{i+q}(x)$. \\
\end{enumerate}
\begin{figure}
\begin{subfigure}{\linewidth}
  \centering
  \includegraphics[width=.8\linewidth]{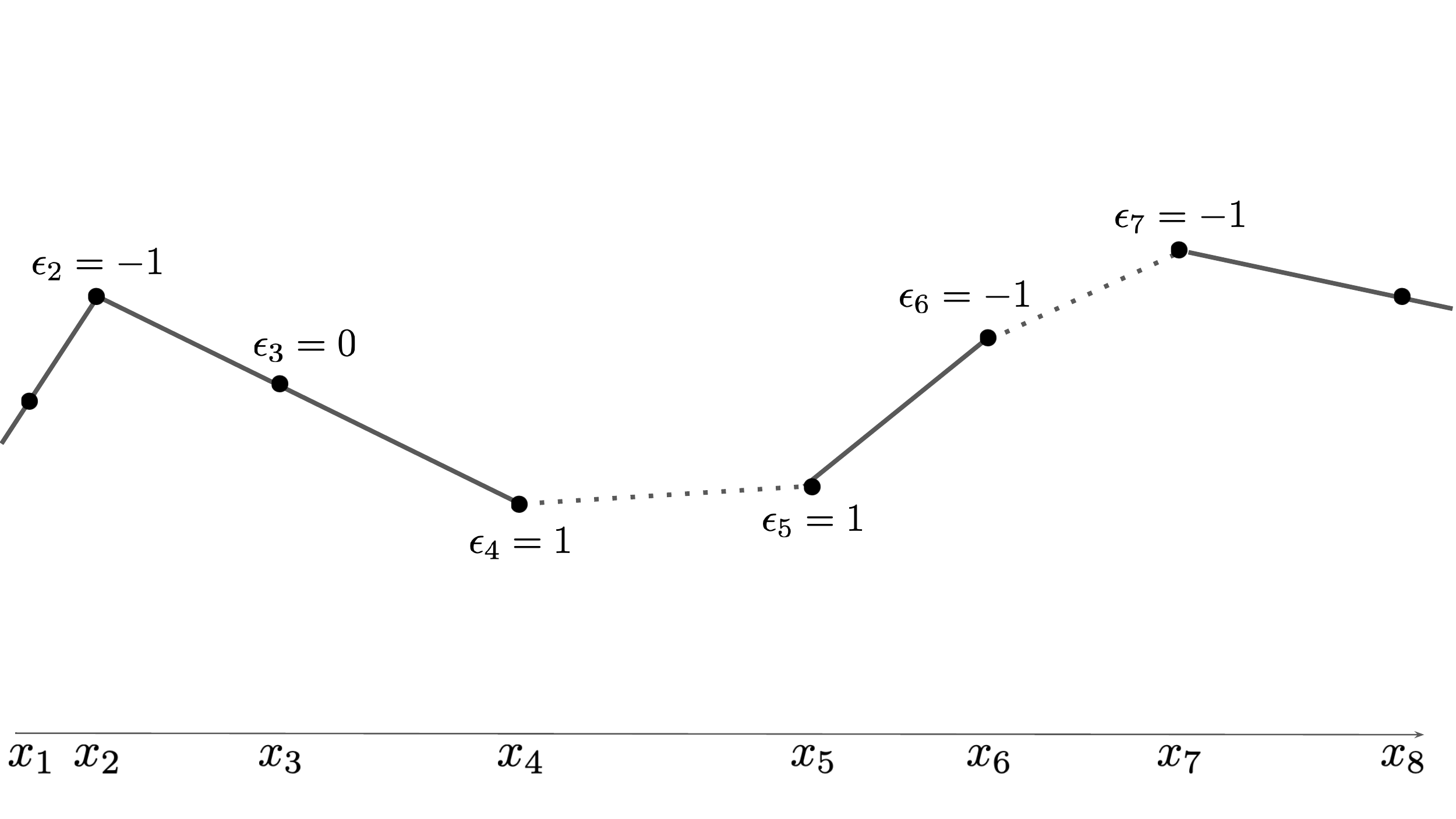}
  \caption{Step 4}%\label{fig:image3}
\end{subfigure} 
   
\begin{subfigure}{\linewidth}
    \centering
     \includegraphics[width=.8\linewidth]{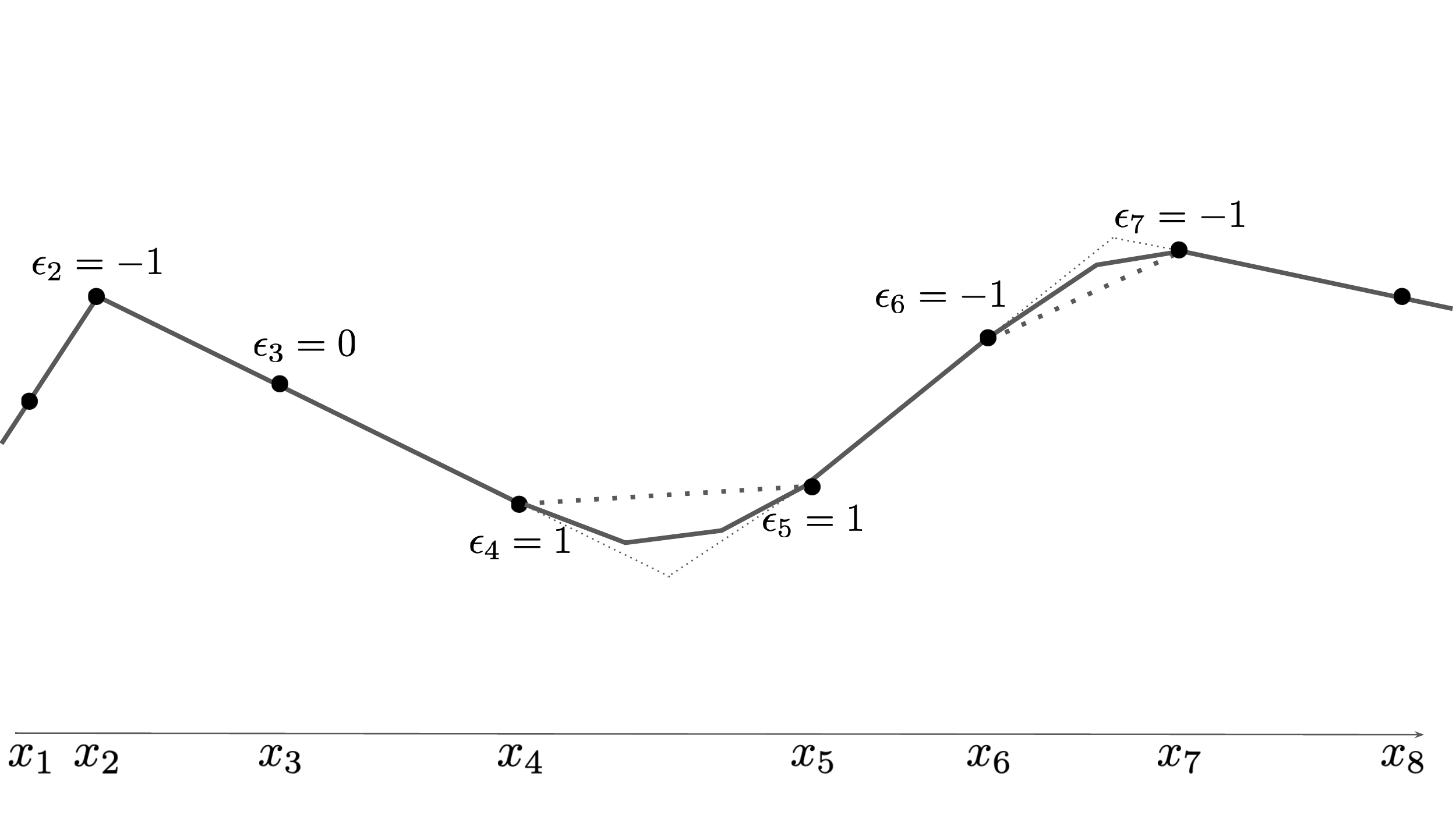}
    \caption{Step 5. One possible choice of a convex interpolant on $(x_4,x_5)$ and of a concave interpolant on $(x_6,x_7)$ is shown. Thin dashed lines are the supporting lines that bound all interpolants below on $(x_4,x_5)$ and above on $(x_6,x_7)$.}
\end{subfigure}
\caption{Steps 4 - 5 for generating $\RR$ from the dataset $\mD$.}
\label{fig:alg2}
\end{figure}

The starting point for the proof of Theorem \ref{T:interpolants} comes from the prior articles \cite{neyshabur2014search,savarese2019infinite,ongie2019function}, which obtained an insightful ``function space'' interpretation of $\RR$ as a subset of $\PLD$. Specifically, a simple computation (cf e.g. Theorem $3.3$ in \cite{savarese2019infinite} and also Lemma \ref{L:ridgeless-pl} below) shows that $f_{\mD}$ achieves the smallest value of the total variation $\TV{Df}$ for the derivative $Df$ among all $f\in \PLD.$ (The function $Df$ is piecewise constant and $\TV{Df}$ is the sum of absolute values of its jumps.) Part of the content of the prior work \cite{neyshabur2014search, savarese2019infinite,ongie2019function} is the following result
\begin{theorem}[cf Lemma $1$ in \cite{ongie2019function} and around equation (17) in \cite{savarese2019infinite}]\label{T:prior}
For any dataset $\mD$ we have
\begin{equation}\label{E:rr-old}
\RR = \set{f\in \PLD~|~\TV{Df}=\TV{Df_{\mD}}}. 
\end{equation}
\end{theorem}
Theorem \ref{T:prior} says that $\RR$ is precisely the space of functions in $\PLD$ that achieve the minimal possible total variation norm for the derivative. Intuitively, functions in $\RR$ are therefore averse to oscillation in their slopes. The proof of this fact uses a simple idea introduced in Theorem $1$ of \cite{neyshabur2014search} which leverages the homogeneity of the ReLU to translate between the regularizer $C(\theta)$, which is positively homogeneous of degree $2$ in the network weights, and the penalty $\TV{Df}$, which is positively homogeneous of degree $1$ in the network function. 

Theorem \ref{T:interpolants} yields strong generalization guarantees uniformly over $\RR$. To state a representative example, suppose $\mD$ is generated by a function $f_*:\R\gives\R$:
\[
y_j = f_*(x_j).
\]
We then find the following 
\begin{corollary}[Sharp generalization on Lipschitz Functions from Theorem \ref{T:interpolants}]\label{C:lip-gen}
Fix a dataset $\mD=\set{(x_i,y_i),\,\, i=1,\ldots,m}$. We have
\begin{equation}\label{E:new-lip-est}
\sup_{f\in \RR}\norm{f}_{\mathrm{Lip}} \leq \norm{f_*}_{\mathrm{Lip}}.    
\end{equation}
Hence, if $f_*$ is $L-$Lipschitz and $x_i=i/m$ are uniformly spaced in $[0,1]$, then
\begin{equation}\label{E:lip-gen}
\sup_{f\in \RR}\sup_{x\in [0,1]}\abs{f(x)-f_*(x)}\leq \frac{2L}{m}.    
\end{equation}
\end{corollary}
\begin{proof}
Observe that for any $i=2,\dots, m-1$ and $x\in (x_i,x_{i+1})$ at which $Df(x)$ exists we have
\begin{equation}\label{E:new-grad-est}
\epsilon_i(s_{i-1}-s_i)\leq \epsilon_i(Df(x)-s_i)\leq \epsilon_i(s_{i+1}-s_i).
\end{equation}
Indeed, when $\epsilon_i=0$ the estimate \eqref{E:new-grad-est} follows from property (1b) in Theorem \ref{T:interpolants}. Otherwise, \eqref{E:new-grad-est} follows immediately from the local convexity/concavity of $f$ in property (2). Hence, combining \eqref{E:new-grad-est} with property (1a) shows that for each $i=1,\ldots, m-1$
\[
\norm{Df}_{L^\infty(x_i,x_{i+1})}\leq \max\set{\abs{s_{i-1}},\,\abs{s_i}}.
\]
Again using property (1a) and taking the maximum over $i=2,\ldots,m$ we find
\[
\norm{Df}_{L^\infty(\R)}\leq \max_{1\leq i\leq m-1} \abs{s_i} = \norm{f_{\mD}}_{\mathrm{Lip}}.
\]
To complete the proof of \eqref{E:new-lip-est} observe that for every $i=1,\ldots,m-1$
\[
\abs{s_i} = \abs{\frac{y_{i+1}-y_i}{x_{i+1}-x_i}} = \abs{\frac{f_*(x_{i+1})-f_*(x_i)}{x_{i+1}-x_i}}\leq \norm{f_*}_{\mathrm{Lip}}\quad \Longrightarrow\quad \norm{f_{\mD}}_{\mathrm{Lip}}\leq \norm{f_*}_{\mathrm{Lip}}.
\]
Given any $x\in[0,1]$, let us write $x'$ for its nearest neighbor in $\set{i/m,\, i=1,\ldots,m}$. We find
\[
\abs{f(x)-f_*(x)} \leq \abs{f(x)-f(x')}+\abs{f_*(x')-f_*(x)}\leq \lr{\norm{f}_{\mathrm{Lip}}+\norm{f_*}_{\mathrm{Lip}}}\abs{x-x'}\leq  \frac{2L}{m}.    
\]
Taking the supremum over $f\in \RR$ and $x\in [0,1]$ proves \eqref{E:lip-gen}.
\end{proof}

Corollary \ref{C:lip-gen} gives the best possible generalization error of Lipschitz functions, up to a universal multiplicative constant, in the sense that if all we knew about $f_*$ was that it was $L$-Lipschitz and were given its values on $\set{i/m,\, i=1,\ldots,m}$, then we cannot recover $f_*$ in $L^\infty$ to accuracy that is better than a constant times $L/m$. Further, instead of choosing $x_i=i/m$ the same kind of result holds with high probability if $x_i$ are drawn independently at random from $[0,1]$, with the $2L/m$ on the right hand side replaced by $C\log(m)L/m$ for some universal constant $C>0$. The appearance of the logarithm is due to the fact that among $m$ iid points in $[0,1]$ the the largest spacing between consecutive points scales like $C\log(m)/m$ with high probability. Similar generalization results can easily be established, depending on the level of smoothness assumed for $f_*$ and the uniformity of the datapoints $x_i$.

In writing this article, it at first appeared to the author that the generalization bounds \eqref{E:lip-gen} cannot be directly obtained from the relation \eqref{E:rr-old} of prior work. The issue is that a priori the relation \eqref{E:rr-old} gives bounds only on the global value of $\TV{Df}$, suggesting perhaps that it does not provide  strong constraints on local information about the behavior of ridgeless interpolants on small intervals $(x_i,x_{i+1})$. However, the relation \eqref{E:rr-old} can actually be effectively \textit{localized} to yield the estimates \eqref{E:new-lip-est} and \eqref{E:lip-gen} but with worse constants. The idea is the following. Fix $f\in \RR$. For any $i_*=3,\ldots,m-2$ define the left, right and central portions of $\mD$ as follows:
\begin{align*}
\mD_L&:=\set{(x_i,y_i),~i<i_*},\quad \mD_C:=\set{(x_i,y_i),~ i_*-1\leq i\leq i_*+1},\quad \mD_R:=\set{(x_i,y_i),~ i_*<i}.
\end{align*}
Consider further the left, right, and central versions of $f$, defined by
\begin{align*}
f_L(x)&= \begin{cases} f(x),&~ x<x_{i_*}\\ \ell_{i_*}(x),&~ x>x_{i_*}\end{cases},\qquad
f_R(x)= \begin{cases} f(x),&~ x>x_{i_*}\\ \ell_{i_*}(x),&~ x<x_{i_*}\end{cases}
\end{align*}
and
\[
f_C(x)= \begin{cases} f(x),&~ x_{i*-1}<x<x_{i_*+1}\\ \ell_{i_*-1}(x),&~ x<x_{i_*-1}\\\ell_{i_*}(x),&~ x>x_{i_*+1}\end{cases},
\]
Using \eqref{E:rr-old}, we have $ \TV{Df_{\mD}} = \TV{Df}$. Further,
\begin{align*}
    \TV{Df}\geq  \TV{Df_L}+\TV{Df_C}+\TV{Df_R},
\end{align*}    
which, by again applying \eqref{E:rr-old} but this time to $\mD_{L},\mD_{R}$ and $f_{L},f_{R}$, yields the bound 
\[
  \TV{Df}\geq \TV{f_{\mD_L}}+\TV{Df_C}+\TV{Df_{\mD_R}}.
\]
Using that
\[
 \TV{Df_{\mD}} = \sum_{i=2}^{m}\abs{s_i-s_{i-1}},\quad  \TV{f_{\mD_L}} = \sum_{i=2}^{i_*-2} \abs{s_i-s_{i-1}},\quad \TV{Df_{\mD_R}} = \sum_{i=i_*+2}^{m-1} \abs{s_i-s_{i-1}}
\]
we derive the localized estimate
\[
\abs{s_{i_*+1}-s_{i_*}}+\abs{s_{i_*}-s_{i_*-1}}+\abs{s_{i_*-1}-s_{i_*-2}}\geq \TV{Df_C}  
\]
Note further that
\[
\TV{Df_C}  \geq \max_{x\in (x_i,x_{i+1})} Df(x)-\min_{x\in (x_i,x_{i+1})} Df(x),
\]
where the max and min are taken over those $x$ at which $Df(x)$ exists. The interpolation condition $f(x_i)=y_i$ and $f(x_{i+1})=y_{i+1}$ yields that
\[
\max_{x\in (x_i,x_{i+1})}  Df(x)\geq s_i\qquad \text{and}\qquad \min_{x\in (x_i,x_{i+1})} Df(x)\leq s_i.
\]
Putting together the previous three lines of inequalities (and checking the edge cases $i=2,m-1$), we conclude that for any $i=2,\ldots,m-1$ we have
\[
\norm{Df(x) - s_i}_{L^\infty(x_i,x_{i+1})} \leq \abs{s_{i+1}-s_i}+\abs{s_i-s_{i-1}}+\abs{s_{i-1}-s_{i-2}},
\]
where we set $s_0=s_1$. Thus, proceeding as in the last few lines of the proof of Corollary \ref{C:lip-gen}, we conclude that
\[
\norm{f}_{\mathrm{Lip}}\leq 7\norm{f_*}_{\mathrm{Lip}}
\]
and that therefore for any $x\in [0,1]$ we find
\[
\abs{f(x)-f_*(x)}\leq \frac{14L}{m}
\]
when the datapoints $x_i$ are uniformly spaced. These last two estimates are precisely like those in Corollary \ref{C:lip-gen} but with slightly worse constants.

\section{Acknowledgements}
It is a pleasure to thank Peter Binev, Ron DeVore, Simon Foucart, Leonid Hanin, Jason Klusowski, Rob Nowak, Guergana Petrova,  Pokey Rule, and Daniel Soudry for useful discussions. 

\section{Proof of Theorem \ref{T:interpolants}}\label{S:pf}
\noindent Our proof of Theorem \ref{T:interpolants} is structured as follows. First, we shows that any $f\in \RR$ satisfies properties (1) and (2). This constitutes the majority of the argument and requires several preparatory results, starting with Proposition \ref{P:monotone} and its Corollary \ref{C:in-out}. With these in hand, we derive in Propositions \ref{P:through-i}, \ref{P:i-to-i+1}, and \ref{P:i-to-i+1-op} constraints on the local behavior of $f$ on small intervals of the form $(x_i,x_{i+1})$ or $(x_{i-1},x_{i+1})$. Taken together these Propositions, and several other results, imply properties (1) and (2). The details for this step are around Lemma \ref{L:epsilon-slopes}. Finally, establish in Proposition \ref{P:one-way} that any $f$ which satisfies properties (1) and (2) belongs to $\RR$. To start, we introduce some notation. For each $f\in \PLD$ and every $x\in \R$, let us write
\[
s_{\mathrm{in}}(x)=s_{\mathrm{in}}(f,x): =\lim_{\epsilon\gives 0^+} Df(x-\epsilon),\qquad s_{\mathrm{out}}(x)=s_{\mathrm{out}}(f,x):= \lim_{\epsilon\gives 0^+} Df(x+\epsilon)
\]
for the incoming and outgoing slopes of $f$ at $x$. For any $f\in \mathrm{PL}$ the second derivative $D^2f$ is an atomic measure and we have
\[
D^2 f = \sum_{j=1}^{k} c_j \delta_{\xi_j}, \qquad c_j = s_{\mathrm{out}}(f,\xi_j)-s_{\mathrm{in}}(f,\xi_j)
\]
where $\xi_j$ are the points of discontinuity for the derivative $Df$. We will usually supress $f$ from the notation. Thus, $Df$, and in particular $Dz$ for any one layer ReLU network $z$, has a well-defined total variation
\[
\norm{Df}_{TV}:=\sum_{j=1}^k \abs{c_j}.
\]
\noindent Much of the remainder of our proof results on the following fundamental observation.

\begin{proposition}\label{P:monotone}
Fix $f\in \RR$. For every $i=1,\ldots,m-1$ and $Df$ is monotone on $(x_i,x_{i+1})$ in the sense that the functions $s_{\mathrm{in}}(f,x)$ and $s_{\mathrm{out}}(f,x)$ are both either non-increasing or non-decreasing for $x\in (x_{i},x_{i+1})$.
\end{proposition}
\begin{proof}
We proceed by contradiction. That is, let us suppose that $f\in \RR$ and that for some $i$ there exist 
\[
x_i\leq\xi_1<\xi_2<\xi_3<\xi_4\leq x_{i+1}
\]
such that $f$ is given by distinct affine functions with slopes $\sigma_j$ when restricted to any of $(\xi_j,\xi_{j+1})$ for $j=1,2,3$ but that the sequence $\sigma_1,\sigma_2,\sigma_3$ is not monotone. Without loss of generality we assume 
\begin{equation}\label{E:non-mon}
    \sigma_1 ,\sigma_3 < \sigma_2.
\end{equation}
\begin{figure}
    \centering
    \includegraphics[width=\linewidth]{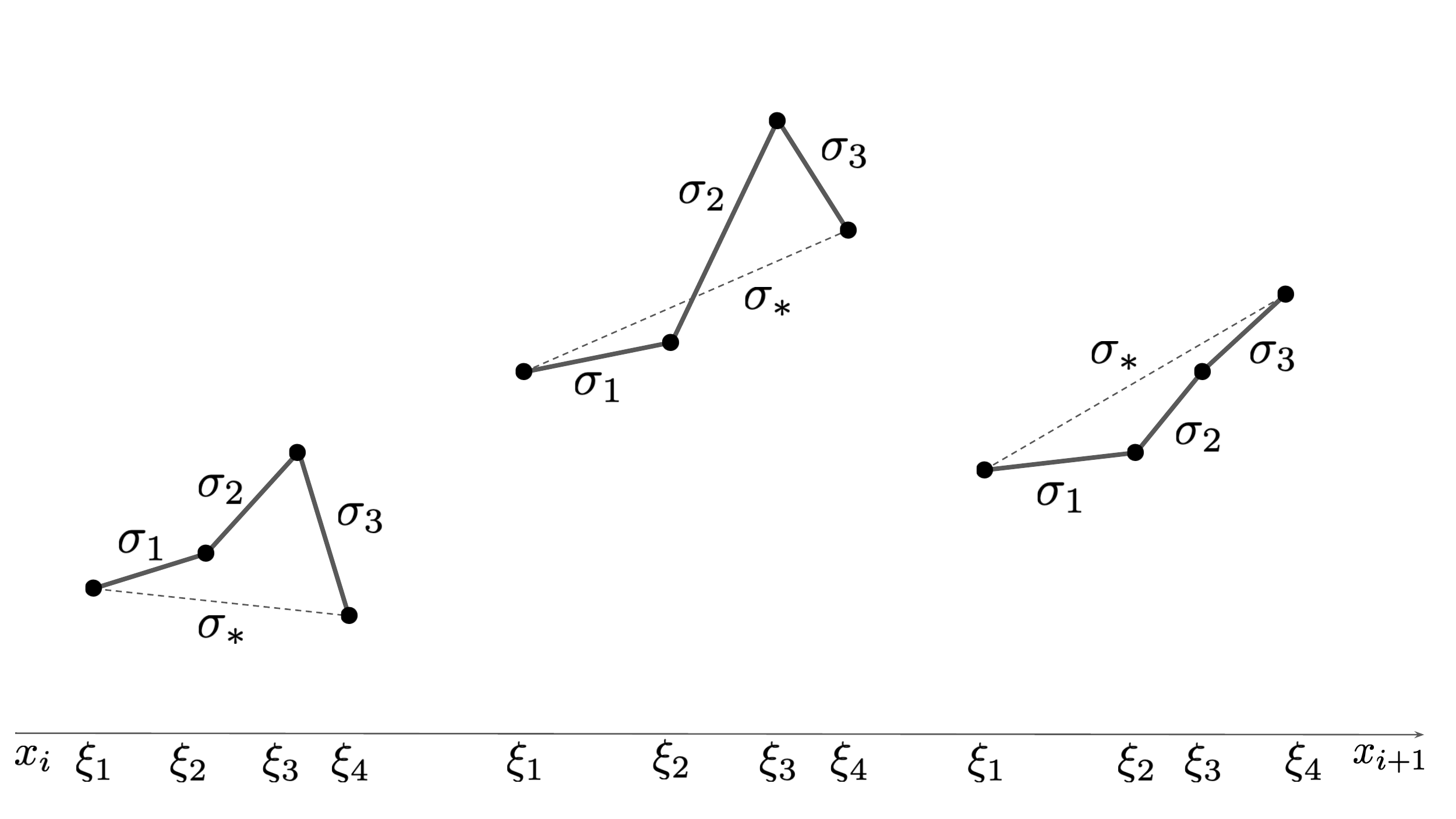}
    \caption{The three possible relative configurations for $\sigma_j,\sigma_*$ are shown. On the left, $\sigma_3<\sigma_*\leq \sigma_1<\sigma_2$. In the center $\sigma_1,\sigma_3<\sigma_*<\sigma_2$. On the right, $\sigma_1<\sigma_*<\sigma_3<\sigma_2$.}
    \label{fig:monotone}
\end{figure}
In particular, for all $\delta$ sufficiently small, we have
\begin{equation}\label{E:ftv}
  \text{Total Variation of }Df\text{ on } (\xi_1-\delta ,\xi_4+\delta ) =  2\sigma_2-\sigma_1-\sigma_3 + \abs{\sigma_1-\sigin}+\abs{\sigma_3-\sigout},
\end{equation}
where
\[
\sigin := s_{\mathrm{in}}(f,\xi_1)=\lim_{\epsilon\gives 0^+} Df(\xi_1-\epsilon)\]
and 
\[
\sigout := s_{\mathrm{out}}(f,\xi_4)=\lim_{\epsilon\gives 0^+} Df(\xi_4+\epsilon).
\]
Define
\[
\sigma_*:= \frac{f(\xi_4)-f(\xi_1)}{\xi_1-\xi_4} = \frac{\sigma_1(\xi_2-\xi_1)+\sigma_2(\xi_3-\xi_2)+\sigma_3(\xi_4-\xi_3)}{\xi_4-\xi_1}.
\]
Note that the constraint \eqref{E:non-mon} and the fact that $\sigstar$ is a convex combination of $\sigma_j$ guarantees that 
\begin{equation}\label{E:sigs}
    \min\set{\sigma_1,\sigma_3}<\sigma_*<\sigma_2.
\end{equation}
See Figure \ref{fig:monotone} for the three possible cases. Consider $g\in \PLD$ defined as follows:
\[
g(x) = \begin{cases} f(x),&\quad x\in (\xi_1,\xi_4)^c\\
\sigma_* (x-\xi_1)+f(\xi_1),&\quad x\in (\xi_1,\xi_4)
\end{cases}.
\]
The function $g$ represents a "straightening of f" between $\xi_1$ and $\xi_4$, and we will now show that the total variation of $Dg$ on $(\xi_1-\delta ,\xi_4+\delta)$ is strictly smaller than that of $Df$ on the same interval. Since the total variations of $Df$ and $Dg$ agree on $(\xi_1,\xi_4)^c$ this will contradict the minimality of $\TV{Df}$ over $\PLD$. Indeed, considering all possible cases for the relative sizes of $\sigin, \sigout$ and $\sigma_*$ we find for all $\delta$ sufficiently small
\begin{align*}
    \text{Total Variation of }Dg\text{ on } (\xi_1-\delta ,\xi_4+\delta) = \max \set{\abs{2\sigstar-\sigin-\sigout},\abs{\sigin-\sigout}}.
\end{align*}
Combining this with the expression \eqref{E:ftv} for the total variation of $Df$ and the following elementary Lemma completes the proof.
\begin{lemma}
For any $\sigma_1,\sigma_2,\sigma_3,\sigma_*$ satisfying \eqref{E:sigs} we have 
\[
2\sigma_2-\sigma_1-\sigma_3 + \abs{\sigma_1-\sigin}+\abs{\sigma_3-\sigout}> \max \set{\abs{2\sigstar-\sigin-\sigout},\abs{\sigin-\sigout}}.
\]
\end{lemma}
\begin{proof}
We consider all four cases for the maximum on the right hand side. We have
\begin{align*}
   & 2\sigma_2-\sigma_1-\sigma_3+\abs{\sigma_1-\sigin}+\abs{\sigma_3-\sigout} - (2\sigstar-\sigin-\sigout)\\
   &\quad=2(\sigma_2-\sigma_*) + \abs{\sigma_1-\sigin} -(\sigma_1-\sigin)+\abs{\sigma_3-\sigout}-(\sigma_3-\sigout)\\&\quad>0,
\end{align*}
as desired. Similarly,
\begin{align*}
   & 2\sigma_2-\sigma_1-\sigma_3+\abs{\sigma_1-\sigin}+\abs{\sigma_3-\sigout} - (\sigin+\sigout-2\sigstar)\\
   &\quad=2(\sigma_2+\sigma_*-\sigma_1-\sigma_3) + \abs{\sigma_1-\sigin} -(\sigin-\sigma_1)+\abs{\sigma_3-\sigout}-(\sigout-\sigma_3)\\&\quad>0,
\end{align*}
as desired. Further,
\begin{align*}
   & 2\sigma_2-\sigma_1-\sigma_3+\abs{\sigma_1-\sigin}+\abs{\sigma_3-\sigout} - (\sigin-\sigout)\\
   &\quad=2(\sigma_2-\sigma_1) + \abs{\sigma_1-\sigin} -(\sigin-\sigma_1)+\abs{\sigma_3-\sigout}-(\sigma_3-\sigout)\\&\quad>0,
\end{align*}
as desired. Finally,
\begin{align*}
   & 2\sigma_2-\sigma_1-\sigma_3+\abs{\sigma_1-\sigin}+\abs{\sigma_3-\sigout} - (\sigout-\sigin)\\
   &\quad=2(\sigma_2-\sigma_3) + \abs{\sigma_1-\sigin} -(\sigma_1-\sigin)+\abs{\sigma_3-\sigout}-(\sigout-\sigma_3)\\&\quad>0,
\end{align*}
completing the proof.
\end{proof}

\end{proof}

Proposition \ref{P:monotone} shows that any $f\in \RR$ is either convex or concave on any interval of the form $\I{i}{i+1}$. This gives several useful consequences, for example the following
\begin{corollary}[of Proposition \ref{P:monotone}]\label{C:in-out}
Fix $f\in \RR$. For each $i=1,\ldots, m-1$, 
\[
\mathrm{sgn}\lr{\sin{i+1}-s_i}  + \mathrm{sgn}\lr{\sout{i}-s_i} =0.
\]
\end{corollary}
\begin{proof}
Suppose first $\mathrm{sgn}(\sout{i}-s_i)=0$. That is, $\sout{i}=s_i$. By Proposition \ref{P:monotone} we have $s_{\mathrm{out}}(x)$ is monotone for $x\in (x_i,x_{i+1})$. Thus, if there exists $\xi\in (x_i,x_{i+1})$ so that $Df(\xi)>s_i$, then $f(x_{i+1})> (x_{i+1}-x_i)s_i + f(y_i)=y_{i+1}$, contradicting the assumption that $f\in \PLD$. A similar contradiction occurs if there exists $\xi\in (x_i,x_{i+1})$ so that $Df(\xi)<s_i$. Hence, we conclude that $\sin{i+1}=s_i$, as desired. Next, suppose $\sout{i}>s_i$. In particular, there exists $\xi_+\in (x_i,x_{i+1})$ such that 
\[
s_{\mathrm{in}}(f,\xi_+) > s_i.
\]
Since $f$ satisfies $f(x_i)=y_i$ and $f(x_{i+1})=y_{i+1}$ there must exist $\xi_-\in (x_i,x_{i+1})$ such that 
\[
s_{\mathrm{in}}(f,\xi_-) < s_i.
\]
By Proposition \ref{P:monotone}, $s_{\mathrm{in}}(f,\xi)$ is monotone for $\xi \in (x_i,x_{i+1})$. We see by comparing $s_{\mathrm{in}}(f,\xi_\pm)$ that it is in fact non-increasing. Since $x_{i+1}-\delta>\xi_-$ for $\delta$ sufficiently small, we conclude that $\sin{i+1}<s_i$, as desired. The case $\sout{i}<s_i$ is analogous, completing the proof. 
\end{proof}

For the remainder of the proof we fix $f\in\RR$ and show that it must satisfy properties (1) and (2). To prove this, we use Proposition \ref{P:monotone} and Corollary \ref{C:in-out} to derive Propositions \ref{P:neighbors}, \ref{P:ends}, \ref{P:through-i}, and \ref{P:i-to-i+1} that together determine the structure of $f$. Specifically, Propositions \ref{P:neighbors}, \ref{P:ends} and a combination of Propositions \ref{P:through-i} and \ref{P:i-to-i+1} show that $f$ satisfies property (1). Then, a different application of Propositions \ref{P:through-i} and \ref{P:i-to-i+1}, together with the fact that $f$ satisfies property (1), will imply that $f$ satisfies property (2) as well.

\begin{proposition}[$f$ agrees with $f_{\mD}$ on colinear neighbors]\label{P:neighbors}
Fix $i=2,\ldots, m-1$. Suppose $\epsilon_i=0$. Then
\[
\sout{i-1}=\sin{i}=\sout{i}=\sin{i+1}=s_{i-1} = s_{i}.
\]
Hence, $f(x)=f_{\mD}(x)$ for all $x\in \I{i-1}{i+1}$.
\end{proposition}
\begin{proof}
By definition, since $\epsilon_i=0$, we have $s_i=s_{i-1}.$ Suppose for the sake of contradiction that 
\[
\text{at least one of }\sout{i-1},\, \sin{i},\, \sout{i},\, \sin{i+1} \text{ does not equal }s_i.
\]
By Corollary \ref{C:in-out}, this means that either one or both least one of the pairs $(\sout{i-1},\, \sin{i})$ or $(\sout{i},\, \sin{i+1})$ are both not equal to $s_i$. We will suppose without loss of generality that 
\begin{equation}\label{E:zero-as}
\min \set{\sout{i-1},\sin{i}}< s_i< \max \set{\sout{i-1},\sin{i}}.
\end{equation}
Note also that by Corollary \ref{C:in-out} and the fact that $f(x_i)=y_i$ and $f(x_{i+1})=y_{i+1}$ we also have
\begin{equation}\label{E:zero-as2}
\min \set{\sout{i},\sin{i+1}}\leq s_i\leq \max \set{\sout{i},\sin{i+1}}.
\end{equation}
By definition, if $\epsilon_i=0$, then $s_{i-1}=s_i$. By Proposition \ref{P:monotone}, the total variation of $Df$ on $(x_{i-1}-\delta, x_{i+1}+\delta)$ equals, for all $\delta$ sufficiently small,
\begin{align*}
&\abs{\sout{i+1}-\sin{i+1}}+\abs{\sin{i+1}-\sout{i}}+\abs{\sout{i}-\sin{i}}\\
&\qquad+ \abs{\sin{i}-\sout{i-1}}+\abs{\sout{i-1}-\sin{i-1}},
\end{align*}
which is bounded below by 
\begin{align*}
&\abs{\sout{i+1}-\sin{i+1}}+\abs{\sin{i+1}-\sout{i}}+ \abs{\sin{i}-\sout{i-1}}+\abs{\sout{i-1}-\sin{i-1}}.
\end{align*}
Define $g\in \PLD$ to coincide with $f$ on $\I{i-1}{i+1}^c$ and to coincide with $f_{\mD}$ on $\I{i-1}{i+1}$. The total variation of $Dg$ on $(x_{i-1}-\delta, x_{i+1}+\delta)$ equals, for all $\delta$ sufficiently small,
\[
\abs{\sout{i+1}-s_i}+\abs{\sin{i-1}-s_i}.
\]
Using that 
\[
\abs{\sout{i+1}-s_i}\leq \abs{\sout{i+1}-\sin{i+1}}+\abs{\sin{i+1}-s_i}
\]
and
\[
\abs{\sin{i-1}-s_i}\leq \abs{\sin{i-1}-\sout{i-1}}+\abs{\sout{i-1}-s_i},
\]
we find that the difference between the total variation of $Df$ and $Dg$ on $(x_{i-1}-\delta, x_{i+1}+\delta)$ is bounded below by 
\begin{align*}
& \abs{\sin{i+1}-\sout{i}}-\abs{\sin{i+1}-s_i} + \abs{\sin{i}-\sout{i-1}}- \abs{s_i-\sout{i-1}}.
\end{align*}
Note that if $a,c\in \R$ and $\min\set{a,c}\leq b\leq \max\set{a,c}$, then we have
\[
\abs{c-a}-\abs{a-b} = \abs{b-c}.
\]
Hence, using our assumptions \eqref{E:zero-as} and \eqref{E:zero-as2}, we conclude that
\[
\abs{\sin{i}-\sout{i-1}}- \abs{s_i-\sout{i-1}} = \abs{\sin{i}-s_i}>0
\]
and that 
\[
\abs{\sin{i+1}-\sout{i}}-\abs{\sin{i+1}-s_i} = \abs{\sin{i+1}-s_i}\geq 0.
\]
The difference between the total variation of $Df$ and $Dg$ on $(x_{i-1}-\delta, x_{i+1}+\delta)$ is thus strictly positive for all $\delta$ sufficiently small. Since $f,g$ agree on $(x_{i-1},x_{i+1})^c$, we find that $\TV{Dg}<\TV{Df}$, contradicting the minimality of $\TV{Df}$ over $\PLD$.
\end{proof}

\noindent Our next result, Proposition \ref{P:ends}, ensures that $f$ and $f_{\mD}$ agree near infinity. 

\begin{proposition}\label{P:ends}
Suppose $f\in\RR$. Then for $x< x_2$ and $x>x_{m-1}$ we have that $f(x)=f_\mD(x)$.
\end{proposition}
\begin{proof}
We focus on the analysis of $f$ on $(-\infty,x_2)$ since the conclusion on $(x_{m-1},\infty)$ follows by symmetry. To start note that $Df(x)=\sout{1}$ for all $x<x_1$. Indeed, if this were not the case, we could define $g\in \PLD$ to coincide with $f$ on $(x_1,\infty)$ but to have slope $\sout{1}$ on $(-\infty,x_1)$. This $g$ belongs to $\PLD$ and satisfies $\TV{Dg}<\TV{Df}$ since the total variation of its derivative on $(-\infty, x_1+\epsilon)^c$ equals that of $Df$ but the total variation of $Dg$ on $(-\infty, x_1+\epsilon)$ vanishes while that of $f$ is non-zero.

Thus, we see that $\sin{1}=\sout{1}$. Let us now prove that $f(x)=f_\mD(x)$ for $x\in (x_1,x_2)$. This will imply $\sout{1}=s_1$ and will complete the proof. Suppose for the sake of contradiction that $\sin{2}\neq s_1$. Then we have from Corollary \ref{C:in-out} that
\[
\min\set{\sout{1},\sin{2}}<s_1<\max\set{\sout{1},\sin{2}}.
\]
Define $g\in \PLD$ to coincide with $f$ on $(x_2,\infty)$ and with $f_\mD$ on $(-\infty,x_2)$. The total variation of $Dg$ on $(-\infty,x_2+\delta)$ for all $\delta$ sufficiently small is
\[
\abs{\sout{2}-s_1},
\]
whereas the total variation of $Df$ on the same interval is
\[
\abs{\sout{1}-\sin{2}}+\abs{\sin{2}-\sout{2}}.
\]
Since by construction $Df$ and $Dg$ agree on $(x_2,\infty),$ the following claim shows that $\TV{Df}>\TV{Dg}$, contradicting the minimality of $\TV{Df}$ over $\PLD$:
\begin{claim}
Suppose $a,b,c\in\R$ satisfy
\[
\min\set{a,b}<c<\max\set{a,b}.
\]
Then for any $d\in \R$ we have 
\[
\abs{d-c}<\abs{a-b}+\abs{b-d}
\]
\end{claim}
\begin{proof}
Suppose first $a<c<b$. Then
\begin{align*}
    \abs{a-b}+\abs{b-d} - \abs{d-c}&= b-a+\abs{b-d}-\abs{d-c}\\
    &=c-a+\abs{b-d}-(d-b)-\abs{d-c}+(d-c)\\
    &>0,
\end{align*}
as desired. Similarly, suppose $b<c<a$ then 
\begin{align}\label{E:comp}
    \abs{a-b}+\abs{b-d} - \abs{d-c}&= a-b +\abs{b-d}-\abs{d-c}.
\end{align}
If $d\geq c$ then $d>b$ and the right hand side of \eqref{E:comp} becomes
\[
a-b+\abs{b-d}-(d-c) = a-b+d-b-d+c =c-b +a-b>0.
\]
Finally, if $d\leq c$ then the right hand side of \eqref{E:comp} becomes
\[
a-b +\abs{d-b}-(c-d) = a-c + \abs{d-b}-(d-b) >0.
\]
This completes the proof.
\end{proof}
\end{proof}

Proposition \ref{P:ends} allows us to know the ``initial'' and ``final'' conditions $\sin{2}$ and $\sout{m-1}$ for the slopes of $f$. In contrast, Proposition \ref{P:through-i} below allows us to take information about the incoming slope $\sin{i}$ of $f$ at $x_i$ and use the local curvature information $\epsilon_i$ at $x_i$ to constrain the outgoing slope $\sout{i}$. See Figure \ref{fig:through_i}.

\begin{figure}
    \centering
    \includegraphics[width=\linewidth]{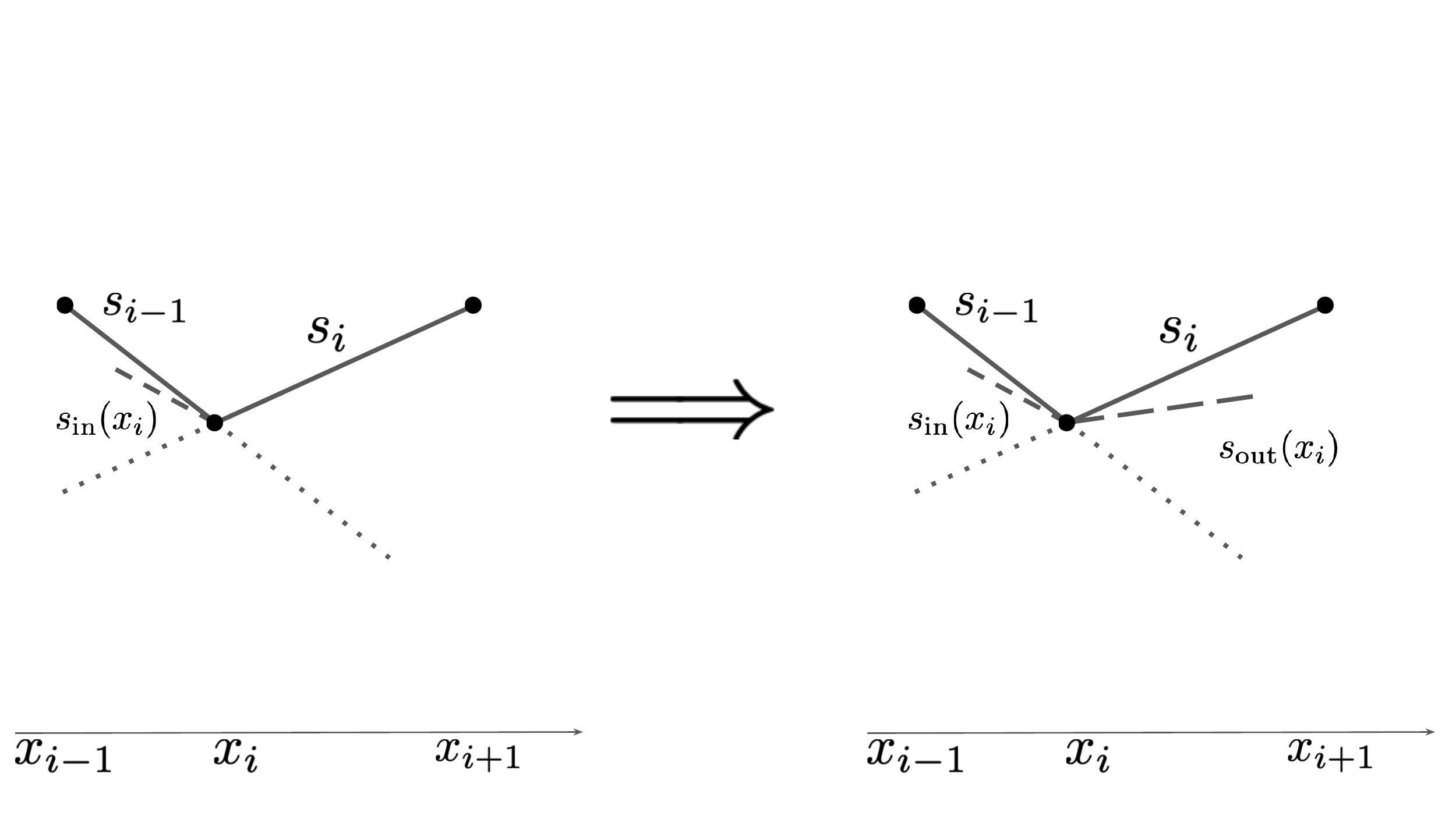}
    \caption{The conclusion of Proposition \ref{P:through-i} when $\epsilon_i=1$.}
    \label{fig:through_i}
\end{figure}

\begin{proposition}[How slope of $f$ changes at $x_i$]\label{P:through-i}
Suppose $\epsilon_i=1$. Then
\begin{equation}\label{E:s-prop-conv}
s_{i-1}\leq \sin{i}\leq s_{i}   \qquad \Longrightarrow \qquad s_{i-1}\leq \sin{i}\leq\sout{i}\leq s_{i}    
\end{equation}
Similarly, suppose $\epsilon_i = -1.$ Then
\begin{equation}\label{E:s-prop-conc}
  s_{i-1}\geq \sin{i}\geq s_{i}       \qquad \Longrightarrow \qquad  s_{i-1}\geq \sin{i}\geq\sout{i}\geq s_{i}    
\end{equation}
\end{proposition}
\begin{proof}
The proof of \eqref{E:s-prop-conc} is identical to that of \eqref{E:s-prop-conv}, and we therefore focus on proving the latter. That is, we fix $i=2,\ldots, m-1$ and assume $\epsilon_i=1$ and suppose that $s_{i-1}\leq \sin{i}\leq s_{i}$. For the sake of contradiction assume also that $\sout{i}> s_i$. By Corollary \ref{C:in-out} we have $\sin{i+1}<s_i$ and therefore the total variation of $Df$ on $(x_i-\epsilon,x_{i+1}+\epsilon)$ is 
\[
\abs{\sout{i+1}-\sin{i+1}} + 2\sout{i} - \sin{i+1}-\sin{i}. 
\]
Consider $g\in\PLD$ defined to be equal to $f$ on $(x_i,x_{i+1})^c$ and to $f_{\mD}$ on $\I{i}{i+1}$. The total variation of $Dg$ on $(x_i-\delta,x_{i+1}+\delta)$ for all $\delta$ sufficiently small is 
\[
\abs{\sout{i+1}-s_i} +s_i - \sin{i}.
\]
The following claim shows that the total variation of $Dg$ on  $(x_i-\delta,x_{i+1}+\delta)$ for all $\delta$ sufficiently small is strictly smaller than that of $Df$. Implies that $\TV{Dg}<\TV{Df}$, which is a contradiction.
\begin{claim}
Suppose $a,b,c,d\in \R$ with $\max\set{a,b}\leq c < d$. Then for all $x\in\R$ we have
\[
\abs{x-b}+2d-a-b > \abs{x-c}+c-a
\]
\end{claim}
\begin{proof}
Since $\abs{x-c}\leq \abs{x-d} + d-c$, we have
\[
\abs{x-b}+2d-a-b - ( \abs{x-c}+c-a )\geq d-b>0.
\]
\end{proof}
Next, again for the sake of contradiction, suppose that we still have $\epsilon_i =1$ and $s_{i-1}\leq \sin{i}\leq s_{i+1}$ but also that $\sout{i}<\sin{i}$. Then, by Corollary \ref{C:in-out} we have $\sin{i+1}>s_i$. Moreover, by Proposition \ref{P:monotone} the total variation of $Df$ on $(x_i-\delta,x_{i+1}+\delta)$ for all $\delta$ small enough is 
\[
\abs{\sout{i+1}-\sin{i+1}} + \sin{i+1}+\sin{i}-2\sout{i}.
\]
Consider $g\in\PLD$ defined to be equal to $f$ for $x\in(x_i,x_{i+1})^c$ but for  $x\in \I{i}{i+1}$ given by 
\[
g(x)=\max\set{(x-x_i)\sin{i} + y_i ,\, (x-x_{i-1})\sin{i+1}+y_{i+1}}.
\]
\begin{figure}
    \centering
    \includegraphics[width=.9\linewidth]{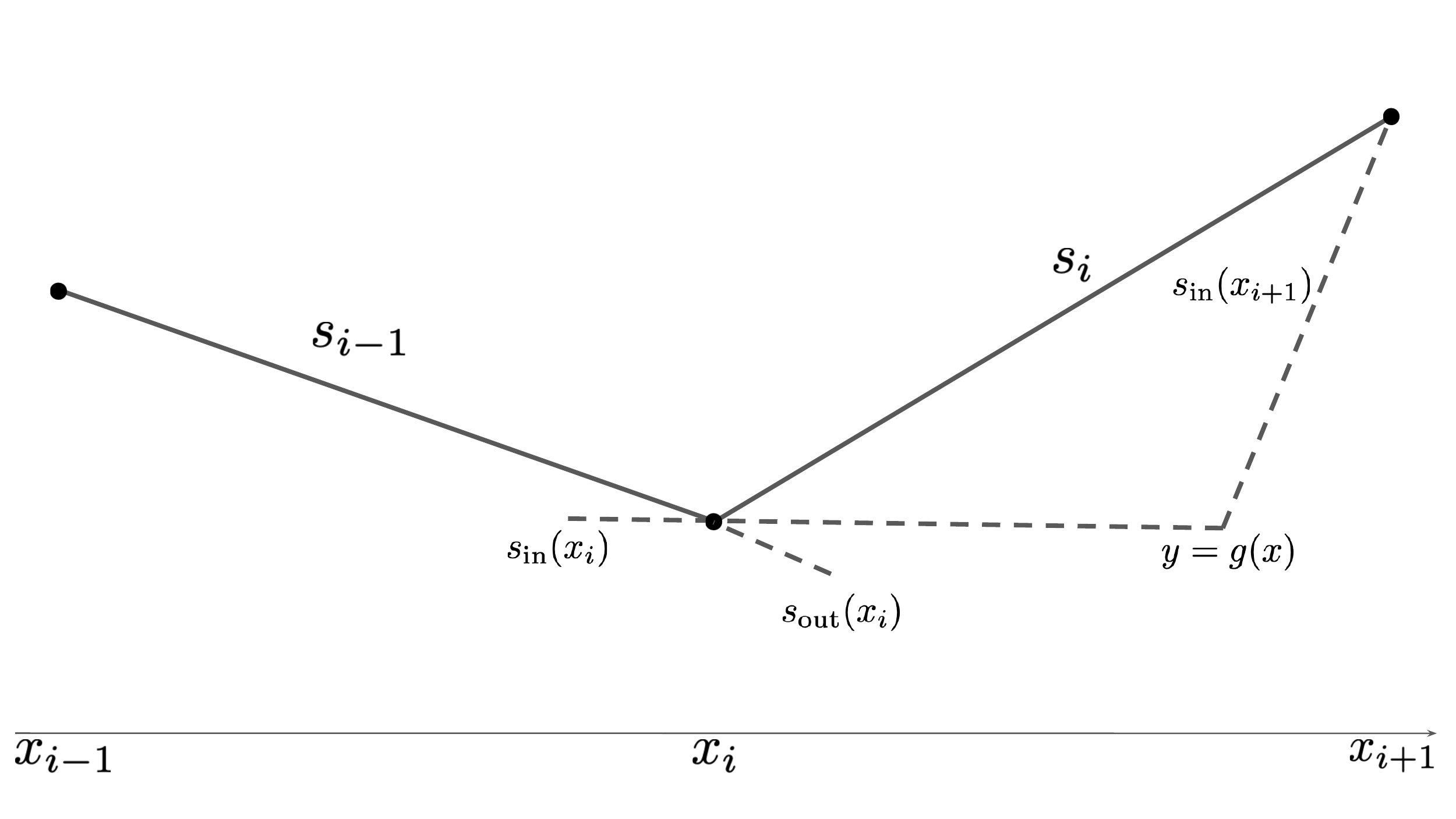}
    \caption{The function $g(x)$ used to derive a contradiction with the assumption that $\sout{i}<\sin{i}$ in Proposition \ref{P:through-i}.}
    \label{fig:through-i-g}
\end{figure}

\noindent See Figure \ref{fig:through-i-g}. The total variation of $Dg$ on $(x_i-\epsilon,x_{i+1}+\epsilon)$ is 
\[
\abs{\sout{i+1}-\sin{i+1}} + \sin{i+1}-\sin{i}.
\]
Therefore the difference between the total variation of $Df$ and $Dg$ on  $(x_i-\delta,x_{i+1}+\delta)$ is 
\[
2(\sin{i}-\sout{i})>0.
\]
Since $f$ and $g$ agree on $(x_{i},x_{i+1})^c$ this contradicts the minimality of $\TV{Df}$ in $\PLD$ and completes the proof of \eqref{E:s-prop-conv}. 
\end{proof}

Proposition \ref{P:through-i} allows us to translate information about the incoming slope $\sin{i}$ to outgoing information about $\sout{i}$. To make use of this, we also need a way to translate between outgoing information $\sout{i}$ and incoming information $\sin{i+1}$. This is done in the following Proposition, whose conclusion is illustrated in Figure \ref{fig:i-to-i+1}.
\begin{figure}
    \centering
    \includegraphics[width=\linewidth]{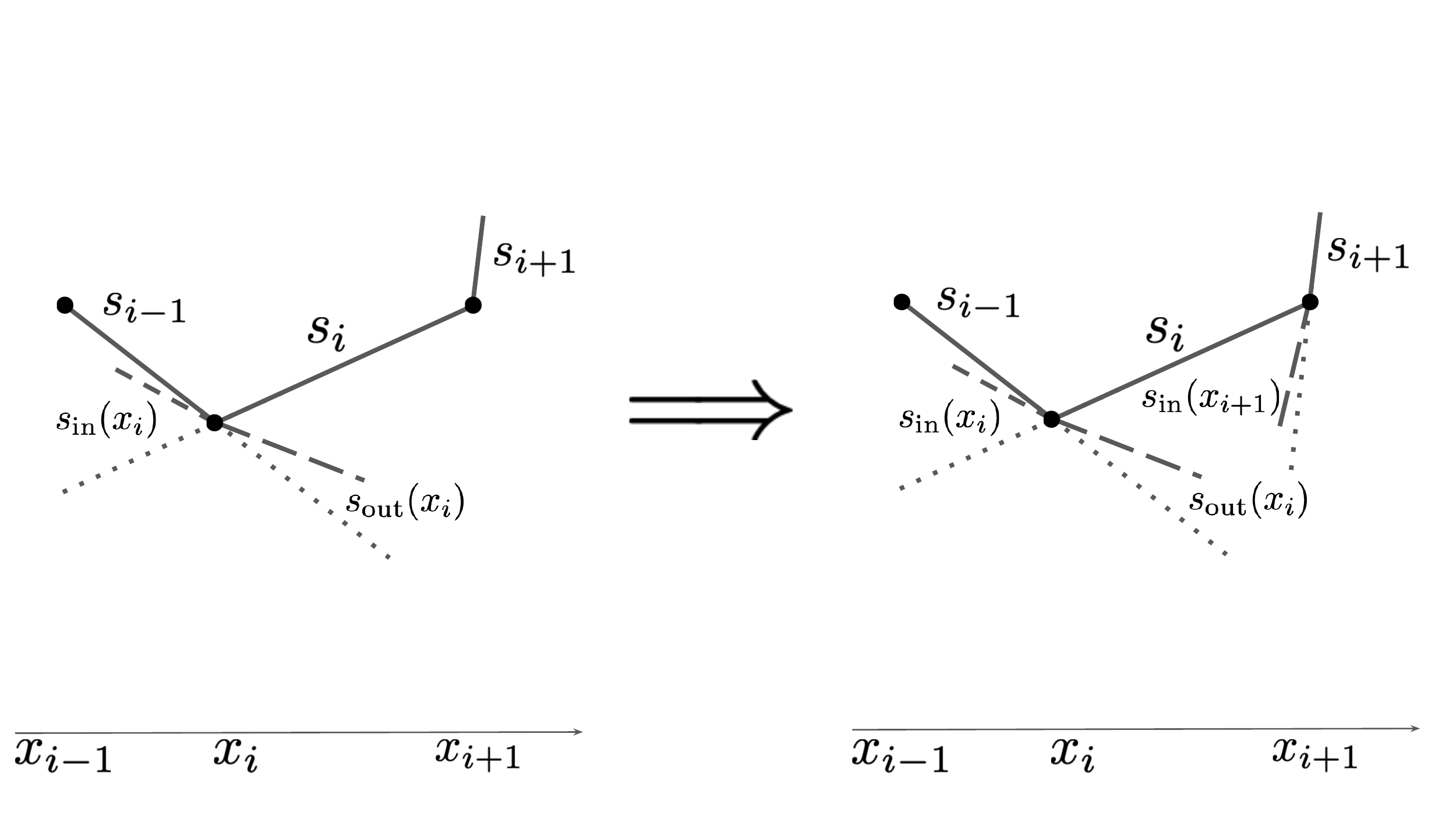}
    \caption{Illustration of the conclusion in Proposition \ref{P:i-to-i+1} when $\epsilon_i=\epsilon_{i+1}=1$.}
    \label{fig:i-to-i+1}
\end{figure}

\begin{proposition}[How slope of $f$ changes between $x_i$ and $x_{i+1}$ when $\epsilon_i,\, \epsilon_{i+1}$ agree]\label{P:i-to-i+1}
If $\epsilon_i=1$ and $s_{i-1}\leq \sin{i}\leq \sout{i}\leq s_i$, then
\begin{align}
 \label{E:same-prop}   \epsilon_{i+1} &= 1\qquad \Longrightarrow\qquad s_i\leq \sin{i+1}\leq s_{i+1}%\\
 %\label{E:op-prop}   \epsilon_{i+1} &= -1\qquad \Longrightarrow\qquad \sout{i}=\sin{i+1}=s_i.
\end{align}
Similarly, if $\epsilon_i=-1$ and $s_{i-1}\geq \sin{i}\geq \sout{i}\geq s_i$, then
\begin{align}
 \label{E:same-prop-2}   \epsilon_{i+1} &= -1\qquad \Longrightarrow\qquad s_i\geq \sin{i+1}\geq s_{i+1}%\\
 %\label{E:op-prop-2}   \epsilon_{i+1} &= 1\qquad \Longrightarrow\qquad \sout{i}=\sin{i+1}=s_i.
\end{align}
\end{proposition}
\begin{proof}
The relation \eqref{E:same-prop-2} follows in the same way as \eqref{E:same-prop}, and so we focus on showing the latter. That is, we suppose $\epsilon_i=\epsilon_{i+1}=1$ and that $s_{i-1}\leq \sin{i}\leq \sout{i}\leq s_i$. Corollary \ref{C:in-out} immediately gives $\sin{i+1}\geq s_i$. To complete the proof of \eqref{E:same-prop} let us suppose for the sake of contradiction that in fact $\sin{i+1}> s_{i+1}$. To derive a contradiction, we need the following observation. 
\begin{lemma}\label{L:i-to-i+1}
Suppose that we have $\epsilon_{i+1}=1$ and $\sin{i+1}>s_{i+1}$. Then we must have $\sout{i+1}<\sin{i+1}$. 
%Suppose that we have $\epsilon_{i}=\epsilon_{i+1}=1$ as well as $s_{i-1}\leq \sin{i}\leq \sout{i}\leq s_{i}$ and $\sin{i+1}>s_{i+1}$. Then we must have $\sout{i+1}<\sin{i+1}$. 
\end{lemma}
\begin{proof}
If $i=m-2$, then the conclusion follows immediately from the fact that by Proposition \ref{P:ends} we have $\sout{i+1}=s_m$. If $i<m-2$, let us suppose for the sake of contradiction that $\sout{i+1}\geq \sin{i+1}$. In particular, we have $\sout{i+1}>s_{i+1}$. Hence, by Corollary \ref{C:in-out} we have
\[
\sin{i+2}<s_{i+1}.
\]
Also by Corollary \ref{C:in-out} since $\sin{i+1}>s_{i+1}>s_i$ we have
\[
\sout{i}<s_i.
\]
\begin{figure}
    \centering
    \includegraphics[width=\linewidth]{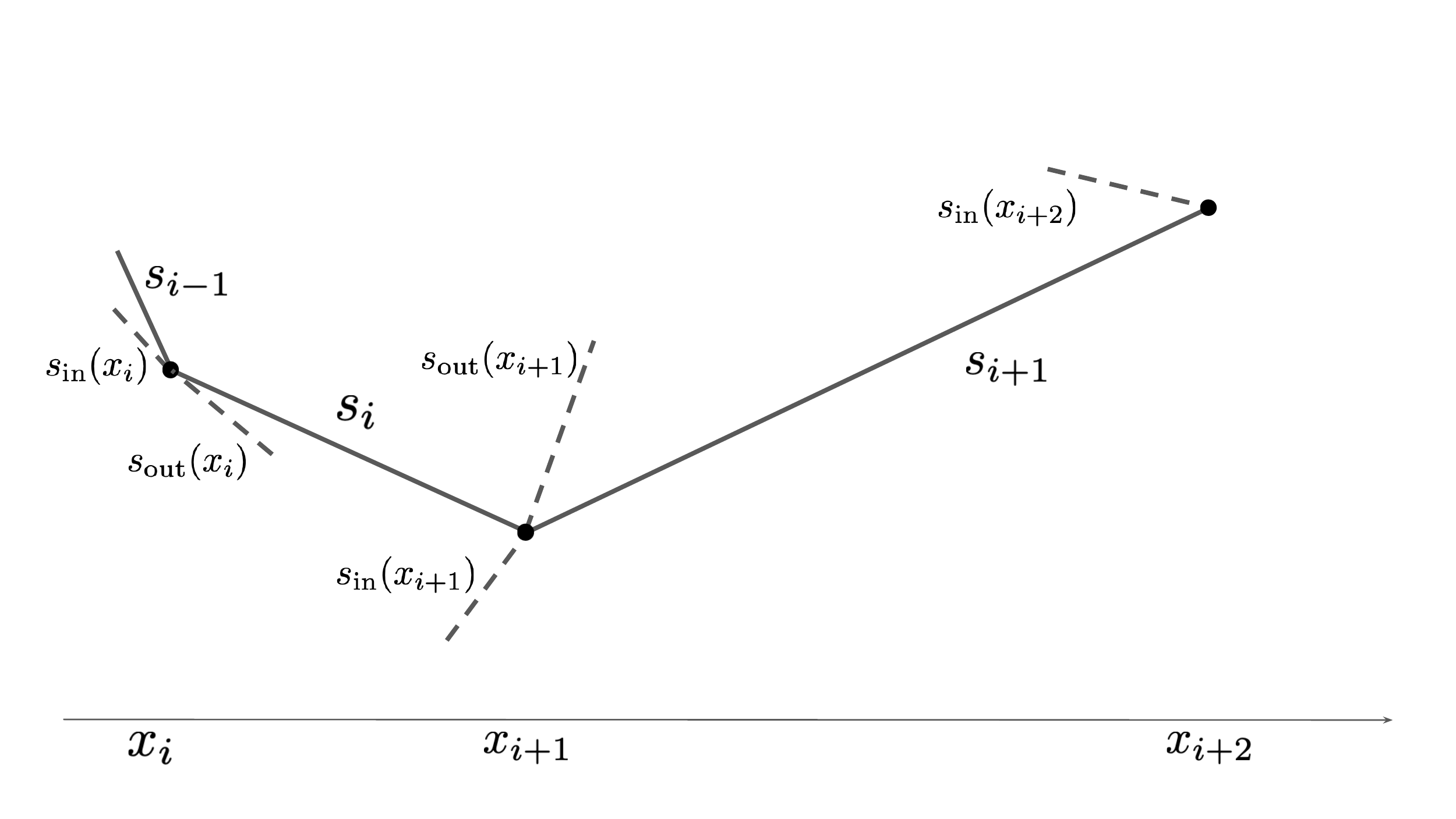}
    \caption{Illustration of hypotheses for contradiction in Lemma \ref{L:i-to-i+1}.}
    \label{fig:i-to-i+1-g}
\end{figure}
See Figure \ref{fig:i-to-i+1-g}. The total variation of $Df$ on $(x_i-\delta,x_{i+2}+\delta)$ for $\delta$ sufficiently small is therefore
\[
\abs{\sout{i+2}-\sin{i+2}} + 2\sout{i+1} - \sin{i+2}-\sin{i} 
\]
Consider $g\in \PLD$ that coincides with $f$ on $(x_{i},x_{i+2})^c$ and with $f_{\mD}$ on $(x_{i},x_{i+2})$. The total variation of $Dg$ on $(x_i-\delta,x_{i+2}+\delta)$ for $\delta$ sufficiently small is 
\[
\abs{\sout{i+2}-s_{i+1}} + s_{i+1}-\sin{i} 
\]
Using that $\abs{\sout{i+2}-s_{i+1}}\leq \abs{\sout{i+2}-\sin{i+2}} + s_{i+1}- \sin{i+2}$, we conclude that the difference between the total variation of $Df$ and $Dg$ is bounded below by 
\[
2\lr{\sout{i+1}-s_{i+1}}>0,
\]
contradicting the minimality of $\TV{Df}$. 
\end{proof}

Returning now to the proof of \eqref{E:same-prop}, we continue to assume that $s_{i-1}\leq \sin{i}\leq \sout{i}\leq s_i$ and $\sin{i+1}>s_{i+1}$. The previous Lemma ensures that therefore
\[
\sin{i+1}>s_*:=\max\set{s_{i+1},\, \sout{i+1}}.
\]
From this last condition we conclude that the total variation of $Df$ on $(x_i-\delta, x_{i+1}+\delta)$ for all $\delta$ sufficiently small is
\[
2\sin{i+1} -\sin{i}-\sout{i+1}.
\]
Consider $g\in\PLD$ defined to be equal to $f$ on $(x_i,x_{i+1})^c$ but on $\I{i}{i+1}$ given by 
\[
g(x) = \max\set{(x-x_{i+1})s_* + y_{i+1},\, (x-x_i)\sout{i}+y_i},\qquad x\in \I{i}{i+1}.
\]
\begin{figure}
    \centering
    \includegraphics[width=\linewidth]{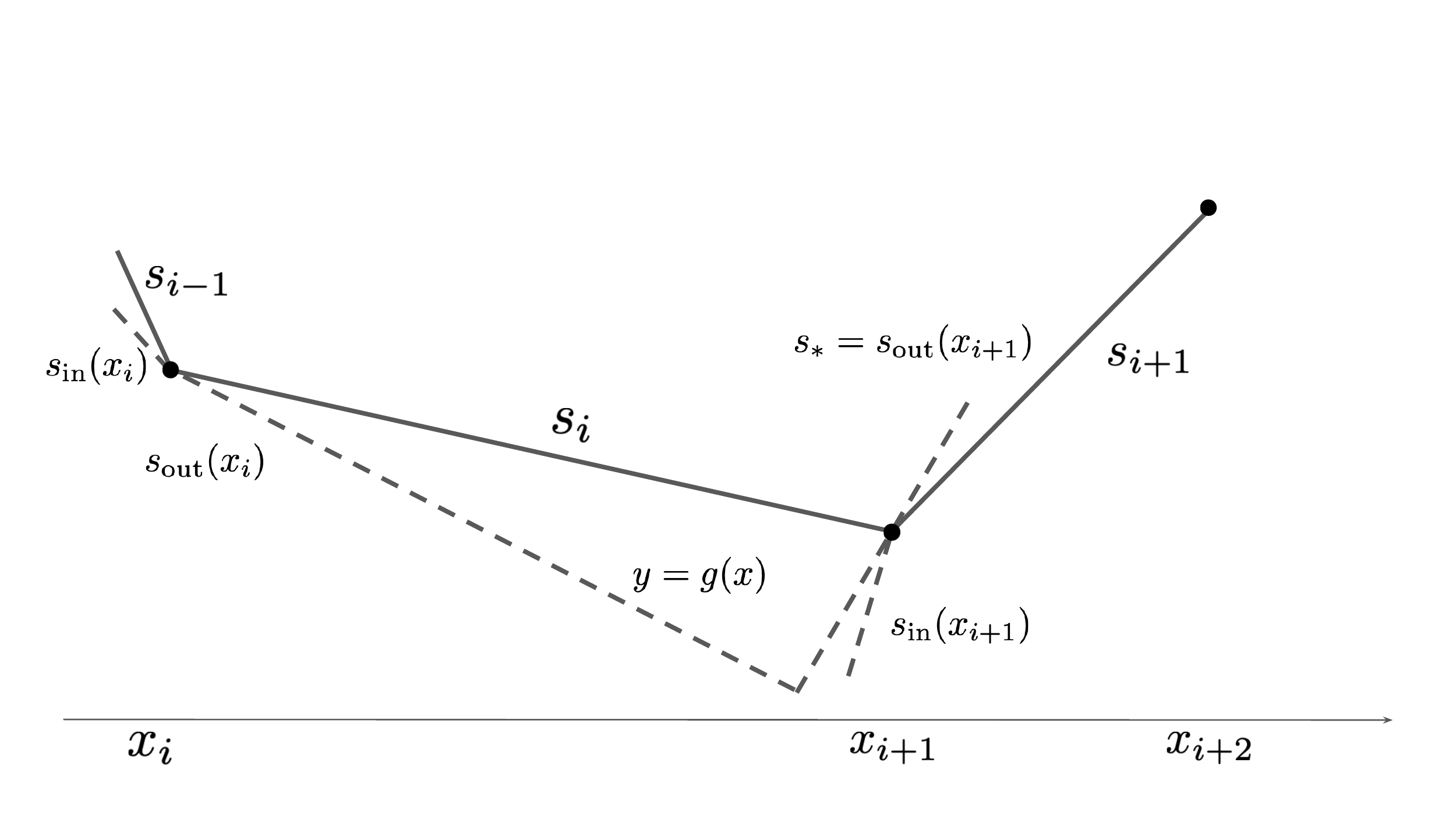}
    \caption{Illustration of the function $g$ used for contradiction at the end of the proof of Proposition \ref{P:i-to-i+1}.}
    \label{fig:i-to-i+1-g2} 
\end{figure}
See Figure \ref{fig:i-to-i+1-g2}. Since $s_*\geq s_{i+1}> s_i$, we find that the total variation of $Dg$ on $(x_i-\delta, x_{i+1}+\delta)$ for all $\delta$ small enough equals
\[
s_*-\sin{i}.
\]
The difference of the total variation of $Df$ and $Dg$ on $(x_i-\delta, x_{i+1}+\delta)$ is therefore given by
\[
\sin{i+1}-s_{i+1} + \sin{i+1}-s_*>0.
\]
This contradicts the minimality of $\TV{Df}$ among $\PLD$ and completes the proof of \eqref{E:same-prop}.

\end{proof}

Proposition \ref{P:i-to-i+1} showed how to use  information about the incoming and outgoing slopes of $f$ at $x_i$ to obtain information on the incoming slop at $x_{i+1}$ if $\epsilon_i=\epsilon_{i+1}$. The following Proposition explains how to do this if instead $\epsilon_i\neq \epsilon_{i+1}$. 

\begin{proposition}[How slope of $f$ changes between $x_i$ and $x_{i+1}$ when $\epsilon_i,\, \epsilon_{i+1}$ disagree]\label{P:i-to-i+1-op}
If $\epsilon_i=1$ and $s_{i-1}\leq \sin{i}\leq \sout{i}\leq s_i$, then
\begin{align}
 \label{E:op-prop}   \epsilon_{i+1} &= -1\qquad \Longrightarrow\qquad \sout{i}=\sin{i+1}=s_i.
\end{align}
Similarly, if $\epsilon_i=-1$ and $s_{i-1}\geq \sin{i}\geq \sout{i}\geq s_i$, then
\begin{align}
\label{E:op-prop-2}   \epsilon_{i+1} &= 1\qquad \Longrightarrow\qquad \sout{i}=\sin{i+1}=s_i.
\end{align}
\end{proposition}
\begin{proof}
Relations \eqref{E:op-prop} and \eqref{E:op-prop-2} are proved in the same way, and so we focus on the former. To show \eqref{E:op-prop}, we suppose $\epsilon_i=1,\, \epsilon_{i+1}=-1$ and that $s_{i-1}\leq \sin{i}\leq \sout{i}\leq s_i$. Suppose for the sake of contradiction that $\sout{i}<s_i$. Then, by Corollary \ref{C:in-out} we have $\sin{i+1}>s_i$. To see why this cannot occur, we give somewhat different arguments depending on whether $\sout{i+1}>s_i$ or $\sout{i+1}\leq s_i$.

Let us first suppose $\sout{i+1}>s_i$. By Corollary \ref{C:in-out} we have $\sin{i+2}<s_{i+1}$. Thus, the total variation of $Df$ on $(x_i-\delta,x_{i+2}+\delta)$ equals
\[
\abs{\sout{i+2}-\sin{i+2}} + \sout{i+2}-\sin{i+2}+\abs{\sout{i+1}-\sin{i+1}} + \sin{i+1}- \sin{i},
\]
which is bounded below by
\[
\abs{\sout{i+2}-\sin{i+2}} + 2\sout{i+1}-\sin{i+2} - \sin{i}.
\]
Define $g\in \PLD$ to coincide with $f$ on $(x_i,x_{i+2})^c$ and with $f_{\mD}$ on $(x_i,x_{i+2})$. The total variation of $Dg$ on $(x_i-\delta,x_{i+2}+\delta)$ is 
\[
\abs{\sout{i+2}-\sin{i+2}}+2s_i - \sin{i+2} - \sin{i}.
\]
Hence, the difference between the total variation of $Df$ and $Dg$ is bounded below by 
\[
2(\sout{i+1}-s_i)>0.
\]
This contradicts the minimality of $\TV{Df}$. Let us now consider the other case: $\sout{i+1}\leq s_i.$ In this case, we have that $\sin{i+1}>\sout{i+1}$. Thus, the total variation of $Df$ on $(x_i-\delta,x_{i+1}+\delta)$ is 
\[
2\sin{i+1} - \sin{i}-\sout{i+1}.
\]
Define $g\in \PLD$ to coincide with $f$ on $(x_i,x_{i+1})^c$ and with $f_{\mD}$ on $(x_i,x_{i+1})$. The total variation of $Dg$ on $(x_i-\delta,x_{i+1}+\delta)$ is 
\[
2s_i - \sout{i+1}-\sin{i}.
\]
Hence, the difference between the total variation of $Df$ and $Dg$ is bounded below by 
\[
2(\sin{i+1}-s_i)>0.
\]
This contradicts the minimality of $\TV{Df}$, completing the proof of Proposition \ref{P:i-to-i+1-op}.
\end{proof}

We are now ready to show that any $f\in \RR$ satisfies (1) and (2). We already know from Propositions  \ref{P:neighbors} and \ref{P:ends} that $f$ satisfies properties (1a) and (1b). In order to check that $f$ satisfies (1c) and (2), we will use the following result. 

\begin{lemma}\label{L:epsilon-slopes}
Suppose $f\in \RR$. For $i=2,\ldots,m-1$ we have
\begin{align*}
\epsilon_i = 1 \qquad &\Longrightarrow\qquad s_{i-1}\leq\sin{i}\leq\sout{i}\leq s_i\\
\epsilon_i = -1 \qquad &\Longrightarrow\qquad s_{i-1}\geq\sin{i}\geq\sout{i}\geq s_i%\\
%\epsilon_i = 0 \qquad &\Longrightarrow\qquad s_{i-1}=\sin{i}=\sout{i}= s_i.
\end{align*}
\end{lemma}
\begin{proof}
We induct on $i$. When $i=2$, we have from Proposition \ref{P:ends} that 
\[
s_1 = \sin{2}. 
\]
%If $\epsilon_2=0$, then $s_1=s_2$ and we conclude from Proposition \ref{P:neighbors} that $\sout{2}=s_2=s_1$, as desired. 
If $\epsilon_2=1$, we may therefore apply Proposition \ref{P:through-i} to conclude that $s_1\leq \sin{1}\leq \sout{2}\leq s_2$, as desired. The case $\epsilon_2=-1$ is similar, completing the base case. Let us now suppose we have  the claim for $2,\ldots,i$. Suppose that $\epsilon_{i+1}=1$ (the case $\epsilon_{i+1}=-1$ is similar). If $\epsilon_{i}\neq 1$, then we conclude from the definition of $\epsilon_{i+1}=1$, the inductive hypothesis, and Propositions \ref{P:neighbors} and \ref{P:i-to-i+1-op} that 
\[
s_i=\sin{i+1}\leq s_{i+1}
\]
Hence, we may apply Proposition \ref{P:through-i} to conclude that $s_i=\sin{i+1}\leq \sout{i+1}\leq s_{i+1}$, as desired. This completes the inductive step and hence the proof of this Lemma.
\end{proof}

Lemma \ref{L:epsilon-slopes} in combination with Corollary \ref{C:in-out} immediately implies that $f$ satisfies property (2). Finally, in combination with Proposition \ref{P:i-to-i+1-op}, Lemma \ref{L:epsilon-slopes} also shows that $f$ satisfies property (1c). This completes the proof that $f\in \RR$ satisfies properties (1) and (2). It  remains to show that every $f$ which satisfies Properties (1) and (2) belongs to $\RR$, which we now establish.

\begin{proposition}\label{P:one-way}
Suppose $f\in \PLD$ satisfies conditions (1) and (2) of Theorem \ref{T:interpolants}. Then, $f$ belongs to $ \mathrm{RidgelessReLU}(\mD)$.
\end{proposition}
\begin{proof}
Define the set $\mI\subseteq\set{1,\ldots, m}$ of discrete inflection points for the connect-the-dots interpolant $f_{\mD}$ (see Figure \ref{fig:i-fig}):
\begin{figure}
    \centering
    \includegraphics[width=\linewidth]{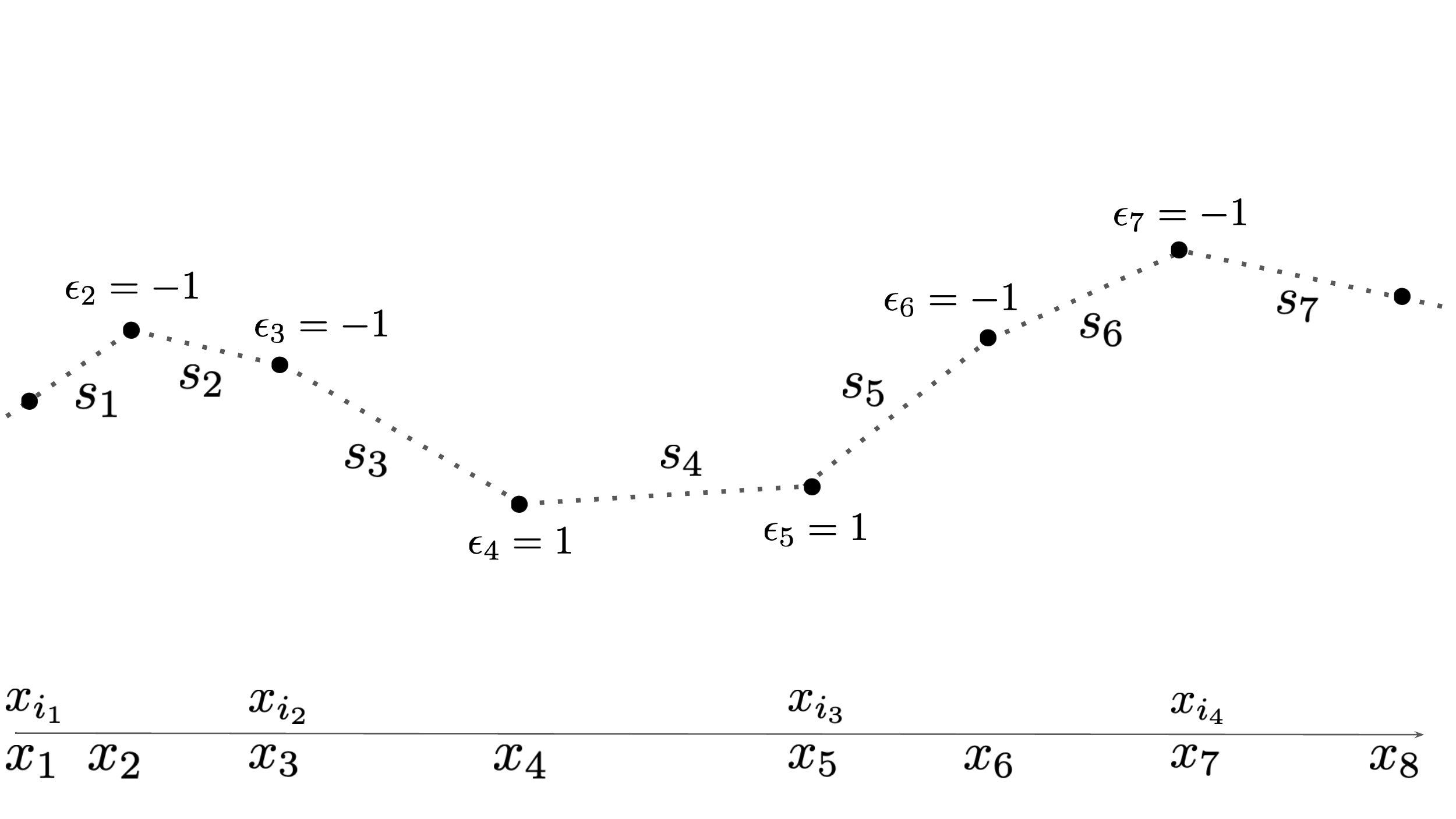}
    \caption{Illustration of the set of discrete inflections points $\mI$ used in Proposition \ref{P:one-way}.}
    \label{fig:i-fig}
\end{figure}
\[
\mI:=\set{i\in \set{2,\ldots, m-2}~|~\epsilon_{i}\neq \epsilon_{i+1}}\cup\set{1,m-1} = \set{i_1=1<i_2<\cdots< i_{\abs{\mI}-1}<i_{\abs{\mI}}=m-1}.
\]
By construction, for each $q=1,\ldots,\abs{\mI}-1$ on the intervals $(x_{i_q},x_{i_q+1}),\ldots, (x_{i_{q+1}}, x_{i_{q+1}+1})$ the sequence of slopes $s_{i_q},\ldots, s_{i_{q+1}}$ of $f_{\mD}$ is either non-increasing or non-decreasing. Hence,
\[
\sum_{j={i_q}}^{i_{q+1}-1} \abs{s_{j}-s_{j+1}} = \abs{s_{i_q}-s_{i_{q+1}}}
\]
and we find
\begin{equation}\label{E:Df-red}
\TV{Df_{\mD}} =\sum_{i=1}^{m-1}\abs{s_{i}-s_{i+1}}=\sum_{q=2}^{\abs{\mI}} \abs{s_{i_q}-s_{i_{q-1}}}.    
\end{equation}
The key observation is 
\begin{equation}\label{E:f-tv}
f\in \PLD \text{ satisfies (1) and (2)}\qquad \Longrightarrow\qquad \TV{Df} = \TV{Df_{\mD}} = \sum_{q=2}^{\abs{\mI}} \abs{s_{i_q}-s_{i_{q-1}}}.     
\end{equation}
Indeed, by property (2), the function $f$ is either convex or concave on any interval of the form $(x_{i_q},x_{i_{q+1}+1})$. Therefore, $Df$ is monotone on any such interval. Thus, we find that
\[
\TV{Df} =\sum_{q=2}^{\abs{\mI}} \abs{s_{\mathrm{out}}(f,x_{i_q})-s_{\mathrm{out}}(f,x_{i_{q-1}})}.
\]
But property (1) guarantees that 
\[
s_{\mathrm{out}}(f,x_{i_q})= s_{i_q}
\]
and for all $q=1,\ldots, \abs{\mI}$, proving \eqref{E:f-tv}. The proof of Proposition \ref{P:one-way} therefore follows from the following result, which was already observed in Theorem 3.3 of \cite{savarese2019infinite}.

\begin{lemma}\label{L:ridgeless-pl} We have
\begin{equation}
    \mathrm{RidgelessReLU}(\mD) = \set{ f \in \mathrm{PL}(\mD) ~\bigg|~ \norm{Df}_{TV} = \norm{Df_{\mD}}_{TV}}
\end{equation}
\end{lemma}
\begin{proof}
Consider any $f\in \RR$. We seek to show that $\TV{Df}\geq \TV{Df_{\mD}}$. Note that for any sequence of points $\xi_1< \cdots < \xi_k$ at which $Df(\xi_j)$ exists, we have
\[
\TV{Df}\geq \sum_{j=1}^{k-1} \abs{Df(\xi_{j+1})-Df(\xi_j)}.
\]
We will now exhibit a set of points where the right hand side equals $\TV{Df_{\mD}}$. To begin, note that by Proposition \ref{P:ends} we have $f(x)=f_{\mD}(x)$ for $x<x_2$ and $x>x_{m-1}$. For all $\xi_{i_1}\in (x_1,x_2)=(x_{i_1},x_{i_1+1})$ and $\xi_{i_{\abs{\mI}}}\in (x_{m-1},x_m)=(x_{i_{\abs{\mI}-1}},x_{i_{\abs{\mI}}})$ we thus have
\[
Df(\xi_{i_1})=s_1,\qquad Df(\xi_{i_{\abs{\mI}}})=s_m.
\]
Further, for any $i=2,\ldots, m-1$ on any interval $(x_i,x_{i+1})$, there exist $\xi_{i,\pm}$ such that $Df(\xi_{i,\pm})$ exist and 
\[
Df(\xi_{i,+}) \geq s_i,\qquad Df(\xi_{i,-})\leq s_i.
\]
In particular, for $q=2,\ldots,\abs{\mI}-1$ we may find $\xi_{i_q}$ satisfying
\[
\xi_{i_q}\in (x_{i_{q}},x_{i_q+1}),\quad \mathrm{sgn} (s_{i_q}-Df(\xi_{i_q}) )= \epsilon_{i_{q+1}}.
\]
As we saw just before this Lemma, for each $i=1,\ldots, \abs{\mI}-1$ we have
\[
\mathrm{sgn}\lr{s_{i_{q+1}}-s_{i_q}}=\epsilon_{i_{q+1}}.
\]
Hence, for each $q=1,\ldots, \abs{\mI} -1$ we conclude
\[
\abs{Df(\xi_{i_q}) -Df(\xi_{i_{q+1}})}\geq \abs{s_{i_q}-s_{i_{q+1}}}.
\]
Thus, 
\[
\TV{Df}\geq \sum_{q=1}^{\abs{\mI}-1} \abs{s_{i_q}-s_{i_{q+1}}} = \TV{Df_{\mD}},
\]
as desired.
\end{proof}

\end{proof}

\bibliography{hanin-bib}{}
\bibliographystyle{alpha}

\end{document}